\pdfoutput=1
\documentclass[12pt]{amsart}
\usepackage[backref=page]{hyperref}
\hypersetup{
  pdfinfo={
    Title={A novel kernel function for Learning Architecture in AI},
    Author={Himanshu Singh},
  }
}
\usepackage{pdfpages}
\usepackage{matlab-prettifier}
\usepackage{tabularx,booktabs}
\newcolumntype{C}{>{\centering\arraybackslash}X} 
\usepackage{longtable}
\usepackage{subcaption}
\usepackage{mathtools}
\usepackage{dsfont}
\usepackage[left=1in,right=1in]{geometry}
\usepackage{amsthm,amsmath,amsbsy,bm}
\usepackage{float}
\usepackage{graphics}
\usepackage{xcolor}
\usepackage{nicefrac}
\usepackage[english]{babel}
\addto\extrasenglish{%
}
\usepackage{xparse,aliascnt,bookmark}
\NewDocumentCommand{\xnewtheorem}{m o m}
 {%
  \IfNoValueTF{#2}
   {\newtheorem{#1}{#3}}
   {%
    \newaliascnt{#1}{#2}%
    \newtheorem{#1}[#1]{#3}%
    \aliascntresetthe{#1}%
    \expandafter\newcommand\csname #1autorefname\endcsname{#3}%
   }%
 
}
\addto\extrasenglish{
    }
\newtheorem{theorem}{Theorem}[section]
\xnewtheorem{definition}[theorem]{Definition}
\xnewtheorem{futuredirection}[theorem]{\textsc{Future Direction}}
\xnewtheorem{lemma}[theorem]{Lemma}
\xnewtheorem{corollary}[theorem]{Corollary}
\xnewtheorem{example}[theorem]{Example}
\xnewtheorem{problem}[theorem]{Problem}
\xnewtheorem{solution}[theorem]{Solution}

\xnewtheorem{proposition}[theorem]{Proposition}
\newcommand\numberthis{\addtocounter{equation}{1}\tag{\theequation}}
\ExplSyntaxOn

\NewDocumentCommand{\pFq}{O{}mmmmm}
 {
  \group_begin:
  \keys_set:nn { hypergeometric } { #1 }
  \hypergeometric_print:nnnnn { #2 } { #3 } { #4 } { #5 } { #6 }
  \group_end:
 }
\NewDocumentCommand{\hypergeometricsetup}{m}
 {
  \keys_set:nn { hypergeometric } { #1 }
 }

\tl_new:N \l_hypergeometric_divider_tl
\tl_new:N \l_hypergeometric_left_tl
\tl_new:N \l_hypergeometric_right_tl

\keys_define:nn { hypergeometric }
 {
  symbol .tl_set:N = \l_hypergeometric_symbol_tl,
  symbol .initial:n = F,
  separator .tl_set:N = \l_hypergeometric_separator_tl,
  separator .initial:n = {},
  skip .tl_set:N = \l_hypergeometric_skip_tl,
  skip .initial:n = 8,
  divider .choice:,
  divider/semicolon .code:n = \tl_set:Nn \l_hypergeometric_divider_tl { \;; },
  divider/bar .code:n = \tl_set:Nn \l_hypergeometric_divider_tl { \;\middle|\; },
  divider .initial:n = semicolon,
  fences .choice:,
  fences/brack .code:n = 
   \tl_set:Nn \l_hypergeometric_left_tl {[}
   \tl_set:Nn \l_hypergeometric_right_tl {]},
  fences/parens .code:n = 
   \tl_set:Nn \l_hypergeometric_left_tl {(}
   \tl_set:Nn \l_hypergeometric_right_tl {)},
  fences .initial:n = brack,
 }

\cs_new_protected:Nn \hypergeometric_print:nnnnn
 {
  {} \sb {#1} \l_hypergeometric_symbol_tl \sb { #2 }
  \left\l_hypergeometric_left_tl
  \genfrac .. 
           {0pt} 
           {} 
           { \__hypergeometric_process:n { #3 } } 
           { \__hypergeometric_process:n { #4 } } 
  \l_hypergeometric_divider_tl
  #5
  \right\l_hypergeometric_right_tl
 }

\cs_new_protected:Nn \__hypergeometric_process:n
 {
  \clist_use:nn { #1 }
   {
    {\l_hypergeometric_separator_tl}
    \mspace { \l_hypergeometric_skip_tl mu }
   }
 }

\ExplSyntaxOff
\usepackage{pgfplots}
\pgfplotsset{width=15cm,compat=newest}
\usepgfplotslibrary{external}
\pgfplotsset{colormap={CM}{rgb(-500)=(.9,.45,.1) color(0)=(red) rgb255(1700)=(15,18,238)}}
\usepackage{wrapfig}
\usepackage{lscape}
\usepackage{rotating}
\usepackage{epstopdf}
\begin{document}
\title[A novel kernel function for Learning Architecture in AI]{\textsc{An appointment with Reproducing Kernel Hilbert Space generated by Generalized Gaussian RBF as $L^2-$measure}}
\author[Himanshu Singh]{Himanshu Singh}
\email{hsingh@uttyler.edu}
\address[Visiting Assistant Professor for Academic Year Aug 2023-May 2024]{Department of Mathematics, The University of Texas at Tyler, TX 75799, USA}
\begin{abstract}
Gaussian Radial Basis Function (RBF) Kernels are the most-often-employed kernels in artificial intelligence and machine learning routines for providing optimally-best results in contrast to their respective counter-parts. However, a little is known about the application of the \emph{Generalized Gaussian Radial Basis Function} on various machine learning algorithms namely, kernel regression, support vector machine (SVM) and pattern-recognition via {neural networks}. The results that are yielded by Generalized Gaussian RBF in the \emph{kernel} sense outperforms in stark contrast to Gaussian RBF Kernel, Sigmoid Function and ReLU Function.

This manuscript demonstrates the application of the \emph{Generalized Gaussian RBF} in the \emph{kernel} sense on the aforementioned machine learning routines along with the comparisons against the aforementioned functions as well. 
Furthermore, we present the explicit description for the reproducing kernel Hilbert Space that is generated by the measure of Generalized Gaussian RBF in $L^2-$measure theoretic sense. Finally, we provide the future directions in terms of eigen-function decomposition and reduced order modeling application of Generalized Gaussian RBF.
\end{abstract}
\maketitle
\section{Introduction}
{A}{rtificial Intelligence} and machine learning algorithms takes the advantage of various important mathematical functions that arises in the function theory. One such function is \emph{Gaussian Radial Basis Function} (GRBF) given as:
\begin{align*}
    g_{\sigma^2}(r)\coloneqq_{\text{def}}\exp\left(-\sigma^2r^2\right);\sigma>0.
\end{align*}
Let $\|\cdot\|_2$ be the usual Euclidean norm on $\mathbf{R}^d$, this function is comfortably synonymous to its kernel notion which is famously called as the GRBF Kernel given as $K_\sigma(\bm{x},\bm{z})\coloneqq_{\text{def}}g_{\sigma^2}(\|\bm{x}-\bm{z}\|_2)$. Explicitly that is
\begin{align*}
    K_\sigma(\bm{x},\bm{z})=\exp\left(-\sigma^2\|\bm{x}-\bm{z}\|_2^2\right);\sigma>0.
\end{align*}
The GRBF Kernel is a building block for various learning architecture such as spatial statistics \protect{\cite{stein1999interpolation}} dynamical system identification \protect{\cite{rosenfeld2019occupation}}, machine learning \protect{\cite{williams2006gaussian}} to name a few. This manuscript extend the idea of GRBF Kernel to what called as the \emph{Generalized GRBF Kernel} (GGRBF Kernel) introduced in \protect{\cite{singh2023new}}.
\begin{definition}
Let $\sigma>0$ and $\sigma_0\geq0$ then the GGRBF Kernel is defined as:
\begin{align*}
    K_{\sigma,\sigma_0}\left(\bm{x}-\bm{z}\right)&\coloneqq_{\text{def}}
    g_{\sigma^2}(\|\bm{x}-\bm{z}\|_2)e^{\left(g_{\sigma_0^2}(\|\bm{x}-\bm{z}\|_2)-1\right)}\\
    &=e^{-\sigma^2\|\bm{x}-\bm{z}\|_2^2}e^{e^{-\sigma_0^2\left(\|\bm{x}-\bm{z}\|_2^2\right)}-1}.\numberthis\label{eq_GGRBF}
\end{align*}
\end{definition}

Note that if $\sigma_0=0$, then we get the traditional GRBF kernel. The GGRBF was introduced in \protect{\cite{karimi2020generalized}} to provide better results in contrast to GRBF results in terms of convergence and stability for interpolation problems on Franke’s test function and Runge’s function or solving the system with Tikhonov regularization and Riley’s algorithm as well. 

Taking much of the inspiration from \protect{\cite{karimi2020generalized}}, the application of GGRBF Kernel was documented in \protect{\cite{singh2023new}} in terms of \emph{support vector machine} (SVM), \emph{kernel regression} and pattern recognition via the \emph{activation function for neural network}. 

These stat-of-the-art methods in learning architecture are leveraged by a peculiar topic from Hilbert function space called as \emph{Reproducing Kernel Hilbert Space} (RKHS) \protect{\cite{aronszajn1950theory}}. The analysis from the RKHS theory for the GRBF Kernel has already been established by \protect{\cite{steinwart2006explicit}} in which answers related to the norms and feature space were answered. However, with the present empirical evidence supporting better results obtained by employing the GGRBF kernel, it becomes important to perform the same investigation for the GGRBF Kernel.

The present paper is organized as follows: we have essential preliminaries of RKHS in \autoref{section_preliminaries}. Then we have results from the function theory in \autoref{section_functionspace} followed by empirical comparison results in \autoref{section_comparisonresults}. 
\section{Notation \& Preliminaries}\label{section_preliminaries}
\subsection{Hypergeomtric Function Notation}
We recall important basic calculus results related to the \emph{Generalized Hypergeometric Function} $\pFq{p}{q}{a_1,a_2,\cdots,a_p}{b_1,b_2,\cdots,b_q}{z}$ \protect{\cite{barnes1906v}}. With the help of Pochhammer symbol \protect{\cite{abramowitz1968handbook}} (rising factorial notation) given as $\left(a\right)_k=\frac{\Gamma(a+k)}{\Gamma(a)}=a(a+1)\cdots(a+k-1)$, then $\pFq{p}{q}{a_1,a_2,\cdots,a_p}{b_1,b_2,\cdots,b_q}{z}$ is given as
\begin{align}
    \pFq{p}{q}{a_1,a_2,\cdots,a_p}{b_1,b_2,\cdots,b_q}{z}\coloneqq_{\text{def}}\sum_{l=0}^\infty\frac{\prod_{i=1}^p\left(a_i\right)_l}{\prod_{i=1}^p\left(b_i\right)_l}\frac{z^l}{l!}.
\end{align}
\begin{example}
 We have the summation $\sum_{l=0}^\infty\frac{1}{\left(l+x\right)^{n+1}}\frac{1}{l!}$ in terms of the Generalized Hypergeometric Function given as:
\begin{align*}
    \sum_{l=0}^\infty\frac{1}{\left(l+x\right)^{n+1}}\frac{1}{l!}=&\sum_{l=0}^\infty\left(\frac{\Gamma(l+x)}{\Gamma\left(l+x+1\right)}\right)^{n+1}\frac{1}{l!}\\
    =&\sum_{l=0}^\infty\left(\frac{\left(x\right)_l\Gamma(x)}{\left(x+1\right)_l\Gamma(x+1)}\right)^{n+1}\frac{1}{l!}\\
    =&~\frac{1}{x^{n+1}}\pFq{n+1}{n+1}{x}{x+1}{1}.
\end{align*}
The example presented above will be useful for further great details in the present manuscript. So to avoid heavy-notation-clutter, we write
\begin{align}\label{eq_5_Hypergeomtric_simplenotation}
\pFq{n+1}{n+1}{x}{x+1}{1}\coloneqq_{\text{notation}}\mathcal{F}_{n,x,1}
\end{align}
from now on-wards upon its need. Note that $\mathcal{F}_{n,\infty,1}=e$.
\end{example}
\subsection{Field Notation}
The set of natural numbers in union with $0$ is denoted by $\mathbf{W}$, that is $\mathbf{W}\coloneqq0,1,2,\ldots$. We use Kronecker delta $\delta_{nm}$ on non-negative integers $n$ and $m$ to depict that, $\delta_{nm}=1$ whenever $n=m$ and $\delta_{nm}=0$ if $n\neq m$. We denote a complex number $z=x+iy$ where $x$ and $y\in\mathbf{R}$. With that $z$, its conjugate-part is given as $\overline{z}=x-iy$ along with its absolute value as $|z|^2=z\cdot\overline{z}=x^2+y^2$. We reserve symbol $\mathbf{K}$ to treat with choice of fields on which we will operate upon; in particular $\mathbf{K}$ can either be $\mathbf{R}$ or $\mathbf{C}$. 
\subsection{Tensor Product Notation}\label{subsection_tensorProdNotation}
We recall the \emph{tensor product} between two functions, say $f_1,~f_2:X\to\mathbf{K}$ given as $f_1\otimes f_2:X\times X\to\mathbf{K}$. Then, for all $x,x'\in X$ the tensor product $f_1\otimes f_2$ is defined as $f_1\otimes f_2(x,x')\coloneqq f_1(x)f_2(x')$.
\subsection{Preliminaries}
\begin{definition}
    Let $X=\emptyset$, then a function $k:X\times X\to\mathbf{K}$ is called the kernel on $X$ if there exists a $\mathbf{K}-$Hilbert space $\left(H,\langle\cdot,\cdot\rangle_H\right)$ accompanied by a map $\Phi:X\to H$ such that $\forall x,x'\in X$, we have
    \begin{align}
        k(x,x')=\langle\Phi(x'),\Phi(x)\rangle_H.\label{eq_kernel_H}
    \end{align}
    We regard $\Phi$ as the feature map and $H$ as the feature space of $k$.
    \end{definition}
Now that we have introduced the basic notion from the kernel theory in the definition provided above, we can now comfortably define the building block of this paper: \emph{Reproducing Kernel Hilbert Space, RKHS}.
\begin{definition}
    Let $X=\emptyset$ and $\left(H,\langle\cdot,\cdot\rangle_H\right)$ be the Hilbert function space over $X$.
    \begin{enumerate}
        \item The space $H$ is called as the \emph{\textbf{reproducing kernel Hilbert space}} (RKHS) if $\forall x\in X$, the evaluation functional $\mathcal{E}_x:H\to\mathbf{K}$ defined as $\mathcal{E}_x(f)\coloneqq f(x),~f\in H$ is continuous. 
    \end{enumerate}
\end{definition}
\begin{definition}\label{defintion_RK}
    A function $k:X\times X\to\mathbf{K}$ is called \emph{\textbf{reproducing kernel}} of $H$ if we have:
        \begin{enumerate}
            \item $k(\cdot,x)\in H~\forall x\in X$, that is $\|k(\cdot,x)\|_H<\infty$, and
            \item $k(\cdot,\cdot)$ has the reproducing property; that is
            \begin{align*}
                f(x)=\langle f,k(\cdot,x)\rangle_H~\forall f\in H\text{~and~} x\in X.
            \end{align*}
        \end{enumerate}
\end{definition}
It is worth-full to mention that the norm convergence yields the point-wise convergence inside RKHS. This fact can be readily learned due to the continuity of evaluation functional. This is demonstrated as follows for an arbitrary $f\in H$ and $\left\{f_n\right\}_n\in H$ with $\|f-f_n\|_H\to0$ as $n\to\infty$, then
\begin{align*}
\lim_{n\uparrow\infty}f_n(x)&=\lim_{n\uparrow\infty}\mathcal{E}_x\left(f_n\right)\\
    &=_{\text{(continuity of $\mathcal{E}_x$)}}\mathcal{E}_x\left(f\right)
    \\
    &=f(x).
\end{align*}
Now, we will state an important theorem from \protect{\cite{aronszajn1950theory}} which dictates the relationship between the reproducing kernel of the RKHS $H$ and the orthonormal basis of it.
\begin{theorem}[\protect{\cite{aronszajn1950theory}}]\label{theorem_aronsjan}
    Let $H$ be an RKHS over an nonempty set $X$, Then $k:X\times X\to\mathbf{K}$ defined as $k(x,x')\coloneqq \langle\mathcal{E}_x,\mathcal{E}_{x'}\rangle_H$ for $x,x'\in X$ is the only reproducing kernel of $H$. Additionally, for some index set $\mathcal{I}$, if we have $\left\{\mathbf{e}_i\right\}_{i\in\mathcal{I}}$ as an orthonormal basis (ONB) then for all $x,x'\in X$, we have
    \begin{align}
        k(x,x')=\sum_{i\in\mathcal{I}}\mathbf{e}_i(x)\overline{\mathbf{e}_i(x')},
    \end{align}
    with an absolute convergence.
\end{theorem}
\section{Function space}\label{section_functionspace}
Let $d\in\mathbf{N}$, $\sigma>0$ and $\sigma_0\geq0$ and $f$ be a holomorphic function $f:\mathbf{C}^d\to\mathbf{C}$, we write first the measure of our interest:
\begin{align}
    d\mu_{\sigma,\sigma_0,d}(\bm{z})\coloneqq e^{-\sigma^2|z|^2}e^{e^{-\sigma_0|z|^2}-1}dV_{\mathbf{C}^d}(\bm{z}),
\end{align}
Here, `$dV_{\mathbf{C}^d}(\bm{z})$' is the usual Lebesgue measure on entire $\mathbf{C}^d$. For $d=1$, we write simply $d\mu_{\sigma,\sigma_0}(z)$ to denote the typical Lebesgue area measure on $\mathbf{C}$.
We now provide the inner product associated with this measure as:
\begin{align}
    \langle f,g\rangle_{\sigma,\sigma_0,\mathbf{C}^d}\coloneqq\mathcal{N}_{\sigma,\sigma_0,d}\int_{\mathbf{C}^d}f(\bm{z})\overline{g(\bm{z})}d\mu_{\sigma,\sigma_0}(\bm{z}).
\end{align}
Here, `$\mathcal{N}_{\sigma,\sigma_0,d}$' is the normalization constant whose value is explicitly given as $\left(\nicefrac{e\sigma^2}{2\pi}\right)^d$. 
Once we have define the inner product for the space, the norm for holomorphic function $f:\mathbf{C}^d\to\mathbf{C}$ is:
\begin{align}\label{eq_6_normf}
    \|f\|_{\sigma,\sigma_0,\mathbf{C}^d}^2\coloneqq&\left(\frac{e\sigma^2}{2\pi}\right)^d
    \int_{\mathbf{C}^d}|f(\bm{z})|^2d\mu_{\sigma,\sigma_0}(\bm{z}).
\end{align}
We write 
\begin{align}\label{eq_8hilbertspace}
    H_{\sigma,\sigma_0,\mathbf{C}^d}\coloneqq\left\{f:\mathbf{C}^d\to\mathbf{C}\text{~s.t.~}\|f\|_{\sigma,\sigma_0,\mathbf{C}^d}<\infty\right\}.
\end{align}
Once we have defined norm in \eqref{eq_6_normf} and the associated Hilbert space in \eqref{eq_8hilbertspace}, we can provide the following formulation which makes the Hilbert space $H_{\sigma,\sigma_0,\mathbf{C}^d}$ as an RKHS.
\begin{lemma}\label{lemma_lemma4}
    For all $\sigma>0$, $\sigma_0\geq0$ and all compact sets $K\subset\mathbf{C}^d$, there exists a positive constant 
    $c_{\sigma,\sigma_0,d}$ such that for all $\bm{z}\in K$ and $f\in H_{\sigma,\sigma_0,\mathbf{C}^d}$, we have
    \begin{align}
        |f(\bm{z})|\leq c_{\sigma,\sigma_0,d}\|f\|_{\sigma,\sigma_0,\mathbf{C}^d}.
    \end{align}
\end{lemma}
\begin{proof}
    Denote $\mathds{B}_{(0,1)}$ as the complex unit ball in $\mathbf{C}$. Define 
    \begin{align*}
        \mathrm{c}_{\sigma,\sigma_0,d}\coloneqq
        \sup_{\bm{z}\in K+\mathds{B}_{(0,1)}^d}\left\{e^{-\sigma^2|\bm{z}|^2}e^{e^{-\sigma_0^2|\bm{z}|^2}-1}\right\}.
    \end{align*}
    In the spirit of \protect{\cite[Lemma-3, Page 4639]{steinwart2006explicit}}, we have
    \begin{align*}
        \prod_{j=1}^dr_j|f(\bm{z})|^2
        \leq&\frac{1}{(2\pi)^d}\prod_{j=1}^dr_j\int_{[0,2\pi]^d}\left|f\left(z_1+r_1e^{i\theta_1},\ldots,z_1+r_de^{i\theta_d}\right)\right|^2d\theta.
    \end{align*}
   Integration of above with respect to $
   \left(r_1,\ldots,r_d\right)\in[0,1]^d$ yields:
   \begin{align*}
       |f(\bm{z})|^2\leq&\frac{1}{(2\pi)^d}\int_{\bm{z}+\mathds{B}_{(0,1)}^d}|f\left(\bm{z}'\right)|^2dV(\bm{z}')\\
       \leq&\frac{\mathrm{c}_{\sigma,\sigma_0,d}}{(2\pi)^d}\int_{\bm{z}+\mathds{B}_{(0,1)}^d}|f\left(\bm{z}'\right)|^2e^{-\sigma^2|\bm{z'}|^2}e^{e^{-\sigma_0^2|\bm{z'}|^2}-1}dV(\bm{z}')\\
       \leq&\frac{\mathrm{c}_{\sigma,\sigma_0,d}}{(e\sigma^2)^d}\|f\|_{\sigma,\sigma_0,d}^2.
   \end{align*}
   In particular, we have $c_{\sigma,\sigma_0,d}=\sqrt{\nicefrac{\mathrm{c}_{\sigma,\sigma_0,d}}{(e\sigma^2)^d}}$ and hence the result is established.
\end{proof}
Establishing \autoref{lemma_lemma4} yields immediately that $H_{\sigma,\sigma_0,\mathbf{C}^d}$ is indeed an RKHS and we state here as an important follow-up corollary.
\begin{corollary}
    The space $\left(H_{\sigma,\sigma_0,\mathbf{C}^d},\langle\cdot,\cdot\rangle_{\sigma,\sigma_0,\mathbf{C}^d}\right)$ is a RKHS for all $\sigma>0$ and $\sigma\geq0$.
\end{corollary}
\subsection{Orthonormal Basis}
We will need following technical result to establish the orthonormal basis (ONB) for $H_{\sigma,\sigma_0,\mathbf{C}^d}$.
\begin{lemma}\label{LEMMA_lemmaiii.3}
    For every $\sigma>0,\sigma_0\geq0$ and $n,m\in\mathbf{W}$, we have
    \begin{align}
        \int_\mathbf{C}z^n\overline{z^m}d\mu_{\sigma,\sigma_0}(z)=\left(\sqrt{\frac{2\pi n!}{e\sigma^{2n+2}}\mathcal{F}_{n,\hat{\sigma},1}}\sqrt{\frac{2\pi m!}{e\sigma^{2m+2}}\mathcal{F}_{m,\hat{\sigma},1}}\right)\delta_{nm}
    \end{align}
    where $\hat{\sigma}=\nicefrac{\sigma_0^2}{\sigma^2}$ and $\mathcal{F}_{n,\hat{\sigma},1}$ is defined in \eqref{eq_5_Hypergeomtric_simplenotation}.
\end{lemma}
\begin{proof}
     Employ the polar coordinate of $z$ to have:
    \begin{align*}
        \int_\mathbf{C}z^n\overline{z^m}d\mu_{\sigma,\sigma_0}(z)
        =&\int_0^{\infty}r^{n+m}e^{-\sigma^2r^2}e^{e^{-\sigma_0^2r^2}-1}rdr\int_0^{2\pi}e^{i\left(n-m\right)\theta}d\theta.\numberthis\label{eq_12_GGRBF_IEEE}
    \end{align*}
The quantity $\int_0^{\infty}r^{n+m}e^{-\sigma^2r^2}e^{e^{-\sigma_0^2r^2}-1}rdr\int_0^{2\pi}e^{i\left(n-m\right)\theta}d\theta$ is $0$ when $n\neq m$. 
Now, assume that $n=m$ in \eqref{eq_12_GGRBF_IEEE}, then:
\begin{align*}
    \int_\mathbf{C}z^n\overline{z^m}d\mu_{\sigma,\sigma_0}(z)
    =&2\pi\int_0^{\infty}r^{2n}e^{-\sigma^2r^2}e^{e^{-\sigma_0^2r^2}-1}rdr\\
    =&\frac{2\pi}{e\left(\sigma^2\right)^{n+1}}\int_0^\infty s^ne^{-s}e^{e^{-\nicefrac{\sigma_0^2}{\sigma^2}s}}ds\\
    =&\frac{2\pi\Gamma(n+1)}{e\left(\sigma^2\right)^{n+1}}\sum_{l=0}^\infty\frac{1}{l!(l\hat{\sigma}+1)^{n+1}}\\
    =&\frac{2\pi n!}{e\sigma^{2n+2}}\mathcal{F}_{n,\hat{\sigma},1}\quad\text{\emph{(use \eqref{eq_5_Hypergeomtric_simplenotation})}}.
\end{align*}
Thus the result prevails.
\end{proof}
In the light of \autoref{theorem_aronsjan}, we have to determine the ONB of $H_{\sigma,\sigma_0,\mathbf{C}^d}$. 
\begin{theorem}\label{theorem_theorem3.4}
    Let $\sigma>0$, $\sigma_0\geq0$ and 
    $n\in\mathbf{W}$. Define $\left\{\mathbf{e}_n\right\}_{n\in\mathbf{W}}:\mathbf{C}\to\mathbf{C}$ by
    \begin{align}\label{eq_ONB_GGRBF}
        \mathbf{e}_n(z)\coloneqq\sqrt{\frac{\sigma^{2n}}{n!\mathcal{F}_{n,\hat{\sigma},1}}}z^n~\forall z\in\mathbf{C}.
    \end{align}
    Then the tensor-product system $\left(\mathbf{e}_{n_1}\otimes\cdots\otimes\mathbf{e}_{n_d}\right)_{n_1,\ldots n_d\geq0}$ forms the ONB of $H_{\sigma,\sigma_0,\mathbf{C}^d}$.
\end{theorem}
\begin{proof}
    We establish our result for $d=1$ for the initial basic understanding. For this, let us show that $\left\{\mathbf{e}_n\right\}_{n\in\mathbf{W}}$ forms an orthonormal system. So, consider $z\in\mathbf{C}$ and let $m,n\in\mathbf{W}$. Then,
    \begin{align*}
        \langle\mathbf{e}_n,\mathbf{e}_m\rangle_{\sigma,\sigma_0}
        =&\frac{e\sigma^2}{2\pi}
        \int_{\mathbf{C}}\mathbf{e}_n(z)\overline{\mathbf{e}_m(z)}d\mu_{\sigma,\sigma_0}(z)\\
        =&\frac{e\sigma^2}{2\pi}\sqrt{\frac{\sigma^{2n}}{n!\mathcal{F}_{n,\hat{\sigma},1}}}\sqrt{\frac{\sigma^{2m}}{m!\mathcal{F}_{m,\hat{\sigma},1}}}\int_{\mathbf{C}}z^n\overline{z^m}d\mu_{\sigma,\sigma_0}(z)\\
        =&\begin{cases}
        1
            &\text{if $n=m$}\\
            0&\text{otherwise}
        \end{cases}\quad\text{\emph{(use \autoref{LEMMA_lemmaiii.3})}}.
    \end{align*}
    The above result concludes that $\left\{\mathbf{e}_n\right\}_{n\in\mathbf{W}}$ is actually an orthonormal system. To this end, we have to establish that it is also complete. So, for this, pick $f\in H_{\sigma,\sigma_0,\mathbf{C}}$ with $f(z)=\sum_{l=0}^\infty a_lz^l$ and observe that
    \begin{align*}
        \left\langle f,\mathbf{e}_n\right\rangle_{\sigma,\sigma_0}
        =&\frac{e\sigma^2}{2\pi}\int_{\mathbf{C}}f(z)
        \overline{\mathbf{e}_n(z)}d\mu_{\sigma,\sigma_0}(z)\\
        =&\frac{e\sigma^2}{2\pi}\sum_{l=0}^\infty a_l\int_{\mathbf{C}}z^l
        \overline{\mathbf{e}_n(z)}d\mu_{\sigma,\sigma_0}(z)\\
        =&\frac{e\sigma^2}{2\pi}\sqrt{\frac{\sigma^{2n}}{n!\mathcal{F}_{n,\hat{\sigma},1}}}
        \sum_{l=0}^\infty a_l\int_{\mathbf{C}}z^l\overline{z^n}d\mu_{\sigma,\sigma_0}(z)\\
        =&\frac{e\sigma^2}{2\pi}\sqrt{\frac{\sigma^{2n}}{n!\mathcal{F}_{n,\hat{\sigma},1}}}\sum_{l=0}^\infty a_l
        \left(\sqrt{\frac{2\pi l!}{e\sigma^{2l+2}}\mathcal{F}_{n,\hat{\sigma},1}}\sqrt{\frac{2\pi n!}{e\sigma^{2n+2}}\mathcal{F}_{n,\hat{\sigma},1}}\right)\delta_{ln}\\
        =&
        \sqrt{\frac{\sigma^{2n}}{n!\mathcal{F}_{n,\hat{\sigma},1}}}a_n\frac{n!\mathcal{F}_{n,\hat{\sigma},1}}{\sigma^{2n}}\\
        =&\left[\sqrt{\frac{\sigma^{2n}}{n!\mathcal{F}_{n,\hat{\sigma},1}}}\right]^{-1}a_n.
    \end{align*}
    Since the constant $\nicefrac{\sigma^{2n}}{n!\mathcal{F}_{n,\hat{\sigma},1}}\neq0$ for any choice of $n$,
    hence, the condition that $\langle f,\mathbf{e}_n\rangle=0$ for all $n\in\mathbf{W}$ yields that $a_n=0$ for all $n\in\mathbf{W}$, which results in conclusion that $f\equiv0$. Therefore, $\left\{\mathbf{e}_{n}\right\}_{n\in\mathbf{W}}$ is complete. Now, we establish these results in $d-$dimensional situation by employing the tensor product notation \autoref{subsection_tensorProdNotation}. To this end, we see that 
    \begin{align*}
\langle\mathbf{e}_{n_1}\otimes\cdots\otimes\mathbf{e}_{n_d},\mathbf{e}_{m_1}\otimes\cdots\otimes\mathbf{e}_{m_d}\rangle_{\sigma,\sigma_0,d}=\prod_{j=1}^d\langle\mathbf{e}_{n_j},\mathbf{e}_{m_j}\rangle_{\sigma,\sigma_0}.
    \end{align*}
     Hence the orthonormality of $\left\{\mathbf{e}_{n_1}\otimes\cdots\otimes\mathbf{e}_{n_d}\right\}_{n_1,\ldots n_d\in,\mathbf{W}^d}$ is established due to the orthonormality of each $\langle\mathbf{e}_{n_j},\mathbf{e}_{m_j}\rangle_{\sigma,\sigma_0}$. We still need to ensure that this $d-$dimensional orthonormal system is complete. 
    Now, observe
    \begin{align*}
        \langle f,\mathbf{e}_{n_1}\otimes\cdots\otimes\mathbf{e}_{n_d}\rangle_{\sigma,\sigma_0,d}
        =&\left(\frac{e\sigma^2}{2\pi}\right)^d\int_{\mathbf{C}^d}f(\bm{z})\overline{\mathbf{e}_{n_1}\otimes\cdots\otimes\mathbf{e}_{n_d}\left(\bm{z}\right)}d\mu_{\sigma,\sigma_0,\mathbf{C}^d}\left(\bm{z}\right)\\
        =&\left(\frac{e\sigma^2}{2\pi}\right)^d\sum_{l_1,\ldots,l_d}^\infty
        a_{l_1,\ldots,l_d}
        \mathds{I}_{l,d}, 
    \end{align*}
    where $\mathds{I}_{l,d}=
\int_{\mathbf{C}^d}\bm{z}^l\left(\mathbf{e}_{n_1}\otimes\cdots\otimes\mathbf{e}_{n_d}\left(\overline{\bm{z}}\right)\right)d\mu_{\sigma,\sigma_0,\mathbf{C}^d}\left(\bm{z}\right)$. We further can simplify $\mathds{I}_d$ as follows:
    \begin{align*}
        \mathds{I}_{l,d}=&
        \int_{\mathbf{C}^d}\bm{z}^l\mathbf{e}_{n_1}(\overline{z_1})\wedge\cdots\wedge\mathbf{e}_{n_d}(\overline{z_d})
        d\mu_{\sigma,\sigma_0}(z_1)\wedge\cdots\wedge d\mu_{\sigma,\sigma_0}(z_d)\\
        =&\prod_{j=1}^d
        \left(\int_\mathbf{C}z_j^{l_j}\mathbf{e}_{n_j}(\overline{z_j})d\mu_{\sigma,\sigma_0}(z_j)\right)\\
        =&\prod_{j=1}^d\left(\int_{\mathbf{C}}z_j^{l_j}\overline{z_j}^{n_j}d\mu_{\sigma,\sigma_0}(z_j)\right)\\
        =&\prod_{j=1}^d\left(\sqrt{\frac{2\pi l_j!}{e\sigma^{2l_j+2}}\mathcal{F}_{l_j,\hat{\sigma},1}}\sqrt{\frac{2\pi n_j!}{e\sigma^{2n_j+2}}\mathcal{F}_{n_j,\hat{\sigma},1}}\right)\delta_{l_jn_j}
        a_{l_1,\ldots,l_d}.
    \end{align*}
    Finally, \begin{align*}
    \left(\frac{e\sigma^2}{2\pi}\right)^d\sum_{l_1,\ldots,l_d}^\infty
        a_{l_1,\ldots,l_d}\mathds{I}_{l,d}=&\left(\prod_{j=1}^d\left[\sqrt{\frac{\sigma^{2n_j}}{n_j!\mathcal{F}_{n_j,\hat{\sigma},1}}}\right]^{-1}\right)a_{n_1,\ldots,n_d}.
    \end{align*}
  The further result for completeness in $d-$dimension follows a routine procedure from single-dimension case as already discussed before.   
\end{proof}
The following theorem provides the reproducing kernel for the Hilbert space $H_{\sigma,\sigma_0,\mathbf{C}^d}$ defined in \eqref{eq_8hilbertspace}.
\begin{theorem}
    For $\sigma>0,~\sigma_0\geq0$ and $\hat{\sigma}=\nicefrac{\sigma^2}{\sigma_0^2}$, the reproducing kernel for the space $H_{\sigma,\sigma_0,\mathbf{C}^d}$ is 
    given as 
    \begin{align}
        K\left(\bm{z},\bm{w}\right)\coloneqq\sum_{n_1,\ldots,n_d=0}^\infty\lambda_{n}\left(\bm{z}\overline{\bm{w}}\right)^n,
    \end{align}
    where muti-index notation is employed: $n=\left(n_1,\ldots,n_d\right)$ and $\lambda_n=\prod_{i=1}^d\frac{\sigma^{2n_i}}{n_i!\mathcal{F}_{n_i,\hat{\sigma},1}}$.
\end{theorem}
\begin{proof}
   We will demonstrate the desired proof as follows:
   \begin{enumerate}
       \item For $\bm{w}\in\mathbf{C}^d$, we will show that $\|K\left(\cdot,\bm{w}\right)\|_{\sigma,\sigma_0,d}<\infty$. 
       \begin{align*}
           \|K\left(\cdot,\bm{w}\right)\|_{\sigma,\sigma_0,d}^2
           =&\left(\frac{e\sigma^2}{2\pi}\right)^d\int_{\mathbf{C}^d}|K\left(\bm{z},\bm{w}\right)|^2d\mu_{\sigma,\sigma_0,d}\left(\bm{z}\right)\\
           =&\left(\frac{e\sigma^2}{2\pi}\right)^d\sum_{n_1,\ldots,n_d}^\infty\sum_{m_1,\ldots,m_d}^\infty\lambda_n\lambda_m\bm{w}^n\bm{\overline{w}}^m
           \int_{\mathbf{C}^d}\bm{z}^n\bm{\overline{z}}^md\mu_{\sigma,\sigma_0,d}\left(\bm{z}\right)\numberthis\label{eq_17_GGRBF}\\
           =&\left(\frac{e\sigma^2}{2\pi}\right)^d\sum_{n_1,\ldots,n_d}^\infty\lambda_n^2|\bm{w}|^{2n}
           \left(\prod_{n_i=1}^d\int_{\mathbf{C}}z_i^{n_i}\overline{z_i}^{m_i}d\mu_{\sigma,\sigma_0}(z_i)\right)
           \numberthis\label{eq_18_GGRBF}\\
           =&\left(\frac{e\sigma^2}{2\pi}\right)^d\sum_{n_1,\ldots,n_d}^\infty\lambda_n^2|\bm{w}|^{2n}\left(\prod_{n_i=1}^d\frac{2\pi n_i!}{\sigma^{2n_i+2}}\mathcal{F}_{n_i,\hat{\sigma},1}\right)\\
           =&\sum_{n_1,\ldots,n_d}^\infty\frac{|\bm{w}|^{2n}}{\prod_{i=1}^d\frac{n_i!}{\sigma^{2n_i}}\mathcal{F}_{n_i,\hat{\sigma},1}}.
       \end{align*}
    We used the result of \autoref{theorem_theorem3.4} from \eqref{eq_17_GGRBF} to\eqref{eq_18_GGRBF}. For all $\bm{w}\in\mathbf{C}^d$, the quantity $\sum_{n_1,\ldots,n_d}^\infty\frac{|\bm{w}|^{2n}}{\prod_{i=1}^d\left(\nicefrac{n_i!}{\sigma^{2n_i}}\right)\mathcal{F}_{n_i,\hat{\sigma},1}}$ achieves convergence. 
    This implies that $\|K\left(\cdot,{\bm{w}}\right)\|_{\sigma,\sigma_0,d}<\infty$ for all $\bm{w}\in\mathbf{C}^d$. Therefore, the kernel function $K(\cdot,\bm{w})\in H_{\sigma,\sigma_0,\mathbf{C}^d}$.
       \item Now, in order to establish the reproducing property of $K(\cdot,\bm{w})$ 
       pick an arbitrary $f=\sum_{n_1,\cdots,n_d}^\infty a_{n_1,\ldots,n_d}\bm{w}^n\in H_{\sigma,\sigma_0,\mathbf{C}^d}$. Then, consider the inner product of $f$ with $K(\cdot,\bm{w})$ as follows:
       \begin{align*}
           \langle f,K(\cdot,\bm{w})\rangle_{\sigma,\sigma_0,{\mathbf{C}}^d}
           =&\left(\frac{e\sigma^2}{2\pi}\right)^d\int_{\mathbf{C}^d}f(\bm{z})\overline{K\left(\bm{z},\bm{w}\right)}d\mu_{\sigma,\sigma_0,\mathbf{C}^d}\left(\bm{z}\right)\\
           =&\left(\frac{e\sigma^2}{2\pi}\right)^d\sum_{n_1,\cdots,n_d}^\infty\sum_{l_1,\cdots,l_d}^\infty a_{n_1,\ldots,n_d}\lambda_{l_1,\ldots,l_d}\bm{w}^n
           \left(\prod_{i=1}^d\frac{2\pi}{e\sigma^2}\frac{\sigma^{2n_i}}
           {n_i!\mathcal{F}_{n_i,\hat{\sigma},1}}\right)\\
           =&\sum_{n_1,\ldots,n_d=0}^\infty a_{n_1,\ldots,n_d}\bm{w}^n\\
           =&f\left(\bm{w}\right).
       \end{align*}
       Hence, the desired result is achieved.
   \end{enumerate}
\end{proof}
The proof in the preceding theorem to demonstrate reproducing kernel nature of $K\left(\bm{z},\bm{w}\right)$ of $H_{\sigma,\sigma_0,\mathbf{C}^d}$ utilizes the basic machinery borrowed from the two-part definition for reproducing kernel given in \autoref{defintion_RK}.
\section{Empirical Evidence and Results Comparison}\label{section_comparisonresults}
\subsection{Kernel Regression}
\subsubsection{Example-1}
Kernel regression of $f(x)=e^{\frac{(1-9x^2)}{4}}+\tan{x}+x^{\frac{1}{6}}+\sin{x^{\vartheta(n)}}$ is performed via both GRBF and \eqref{eq_GGRBF}. Here $\vartheta(n)$ is uniform random distributed number in $(0,1)$, $x=n\in\mathbf{Z}_{101}$.
\begin{figure}[H]
  \centering
  \label{fig:regression-1}
  \includegraphics[width=.45\linewidth,scale=1]{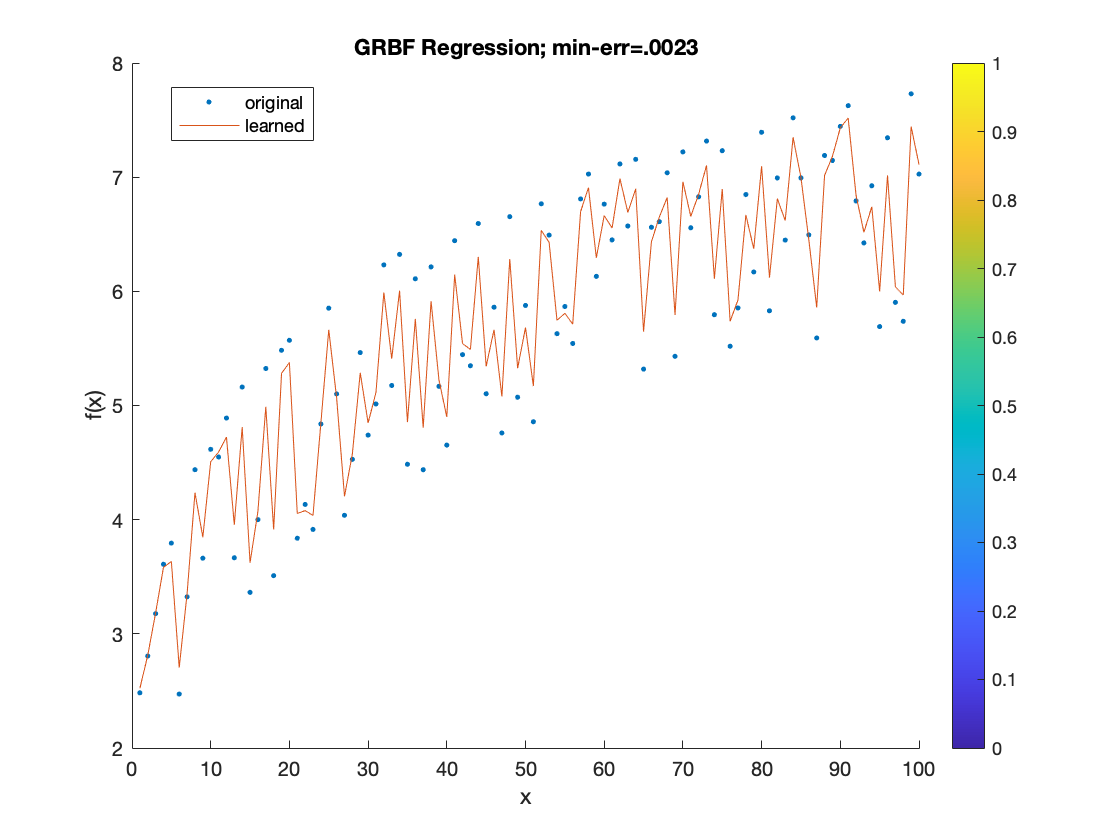}
  \includegraphics[width=.45\linewidth,scale=1]{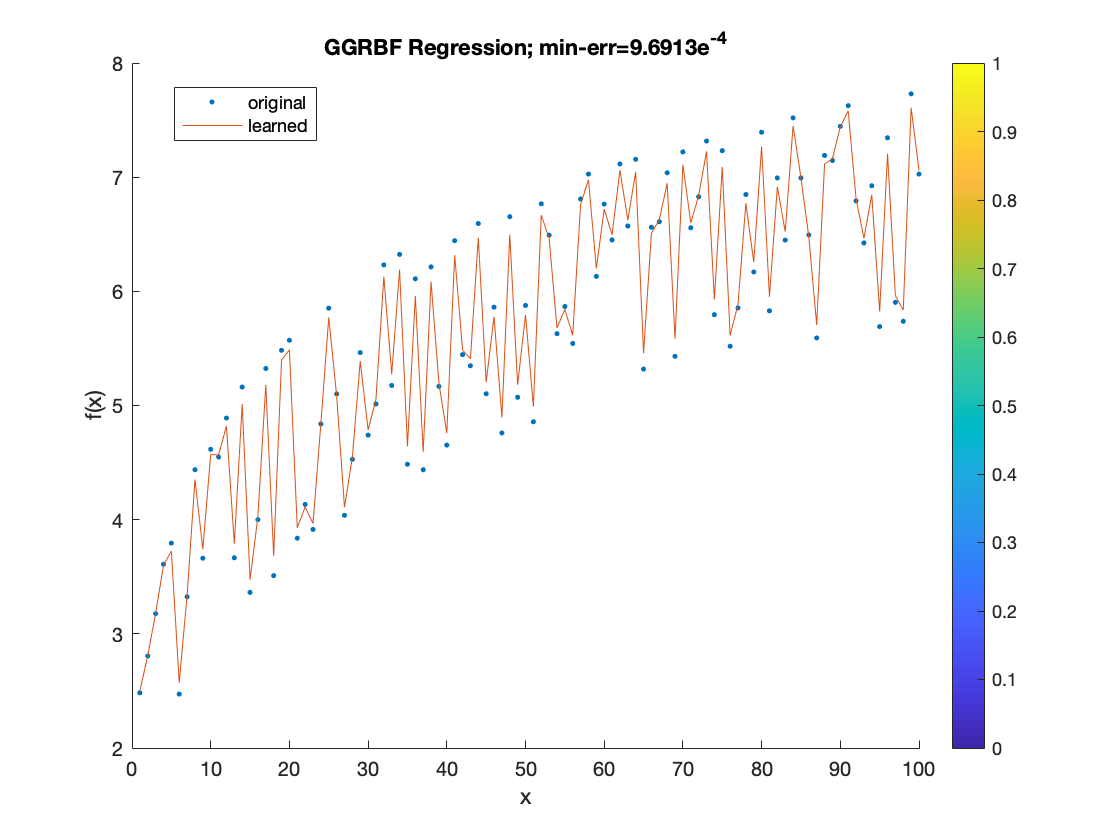}
  \caption{{[\tt{MATLAB}]} Kernel regression of $f(x)=e^{\frac{(1-9x^2)}{4}}+\tan{x}+x^{\frac{1}{6}}+\sin{x^{\vartheta(n)}}$.}
  \label{fig:regression-2}
\end{figure}
\subsubsection{Example-2}
Kernel regression of $f(x)=e^{\sin{x}-\sin{x^2}}+\sqrt{2\pi}|x+\cos{\vartheta(n)}|
$ is performed via both GRBF and \eqref{eq_GGRBF}. Here $\vartheta(n)$ is uniform random distributed number in $(0,1)$, $x=n\in\mathbf{Z}_{101}$.
\begin{figure}[H]
  \centering
  \label{fig:regression-3}
  \includegraphics[width=.45\linewidth,scale=1]{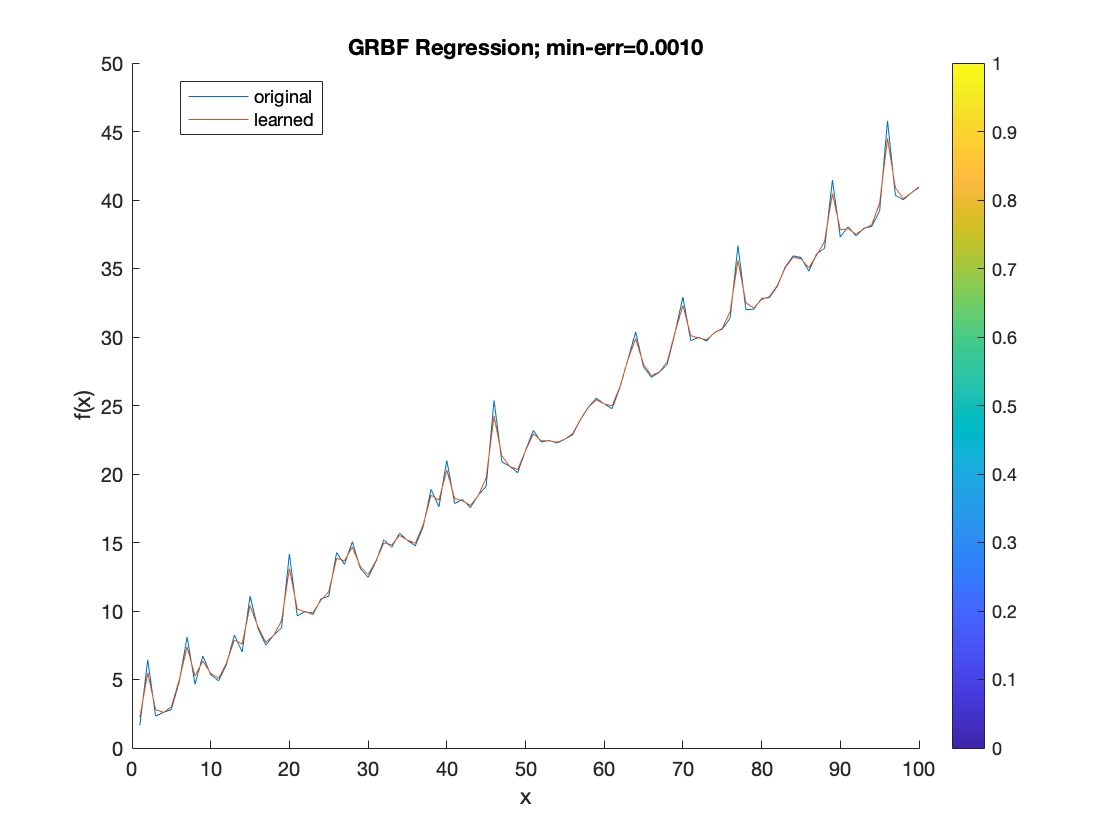}
  \includegraphics[width=0.45\linewidth,scale=1]{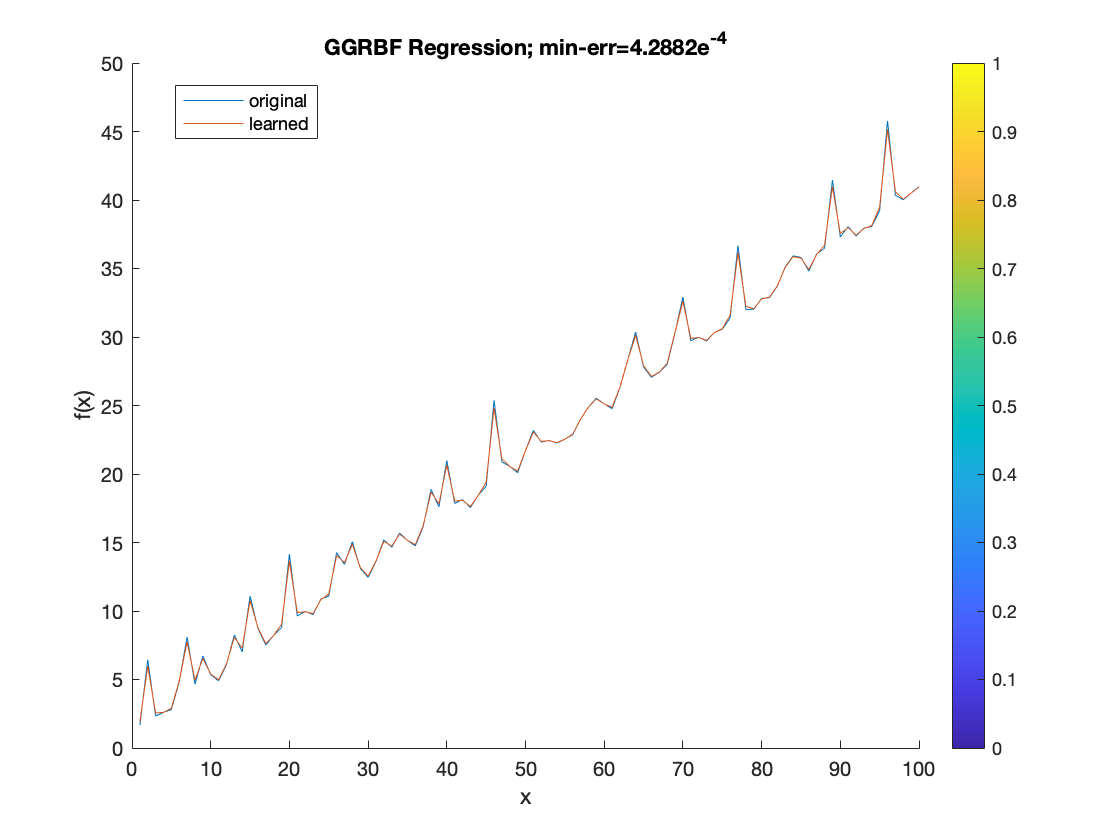}
  \caption{{[\tt{MATLAB}]} Kernel regression of $f(x)=e^{\sin{x}-\sin{x^2}}+\sqrt{2\pi}|x+\cos{\vartheta(n)}|$.}
  \label{fig:regression-4}
\end{figure}
\subsection{Support Vector Machine}
Support vector machine (SVM) is implemented for the the data classification and is performed via different choices of kernels; these kernels includes the traditional \emph{Sigmoid}, GRBF and \eqref{eq_GGRBF}.
\begin{figure}[H]
  \centering
  \label{fig:svm-1}\includegraphics[width=.45\linewidth,scale=1]{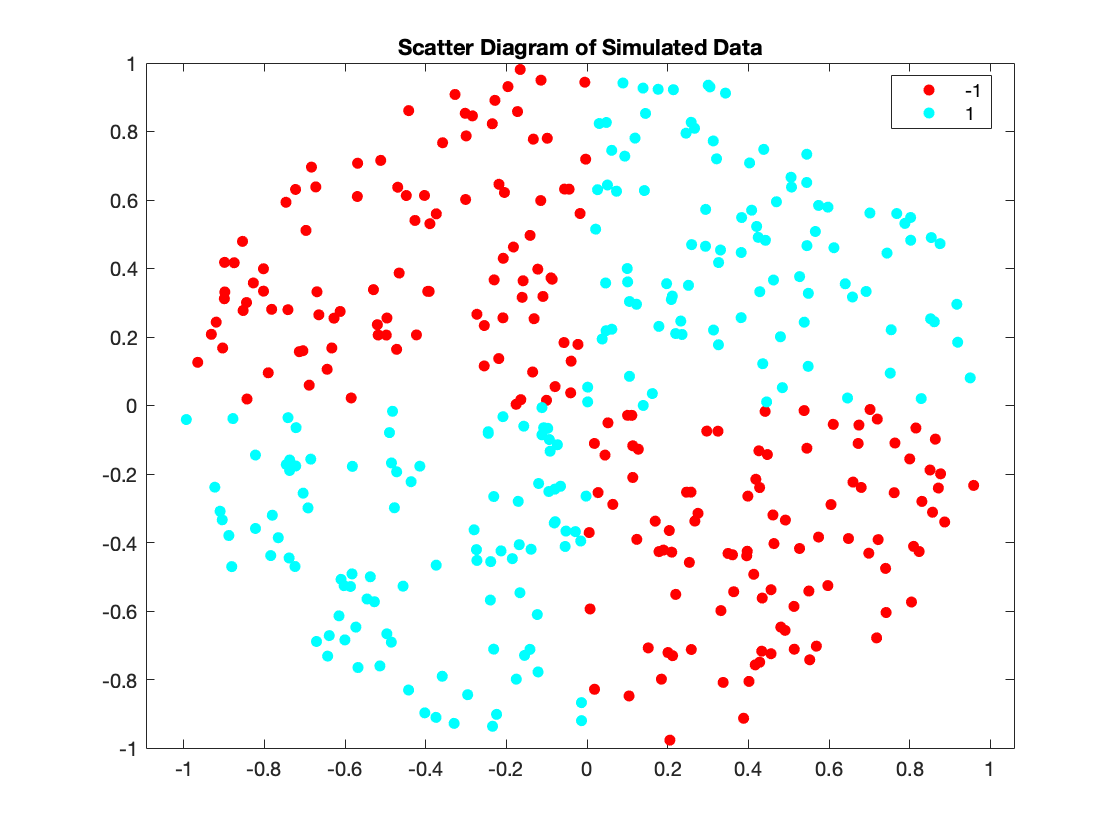}
  \label{fig:svm-2}\includegraphics[width=.45\linewidth,scale=1]{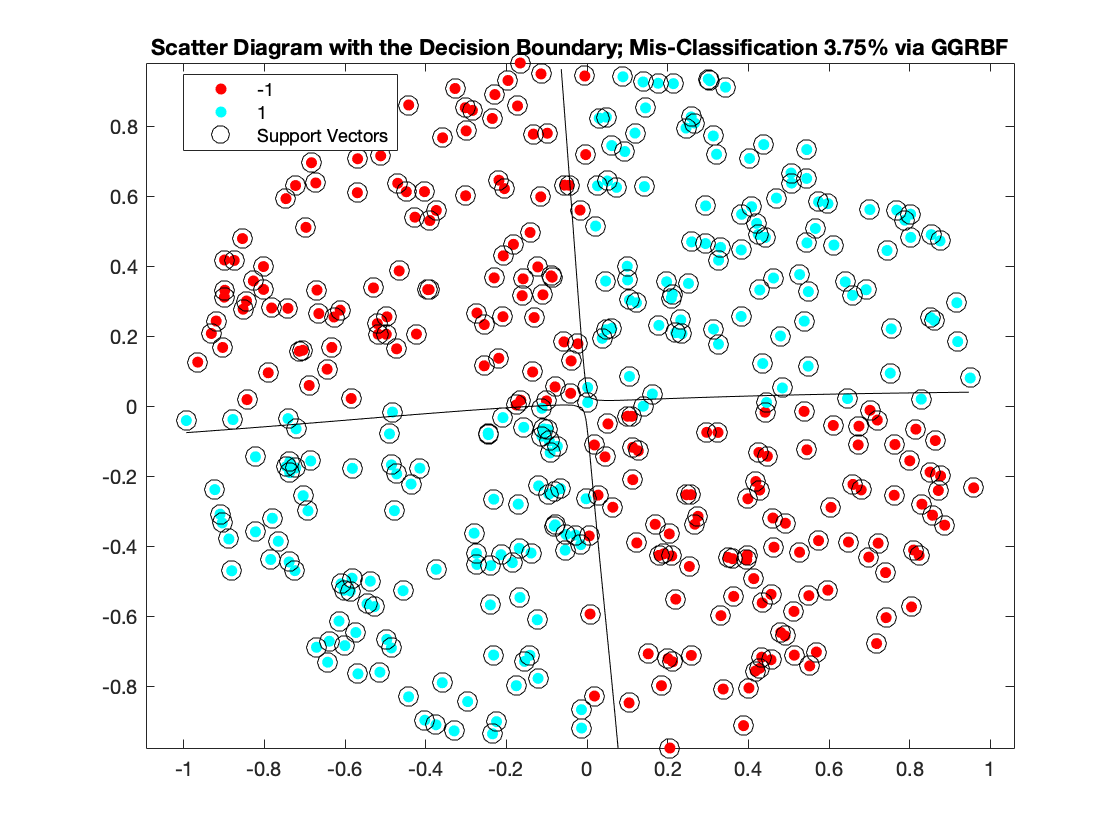}
  \label{fig:svm-3}\includegraphics[width=.45\linewidth,scale=1]{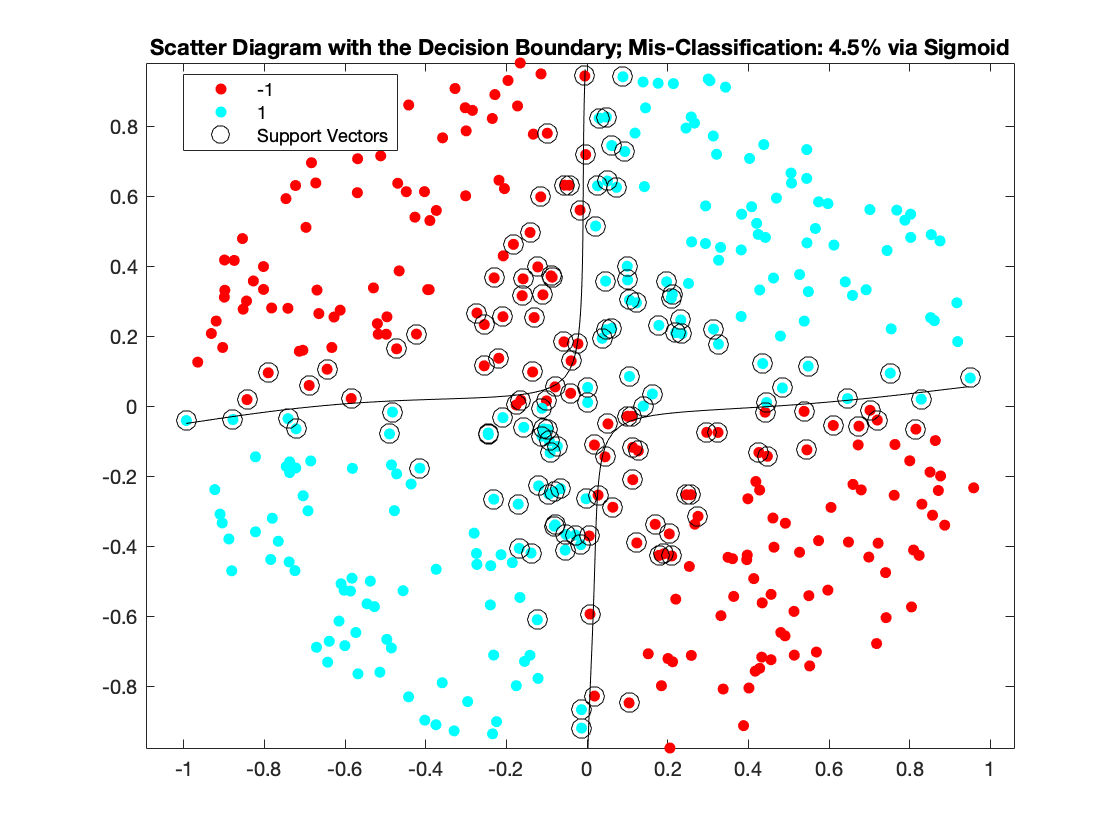} 
  \label{fig:svm-4}\includegraphics[width=.45\linewidth,scale=1]{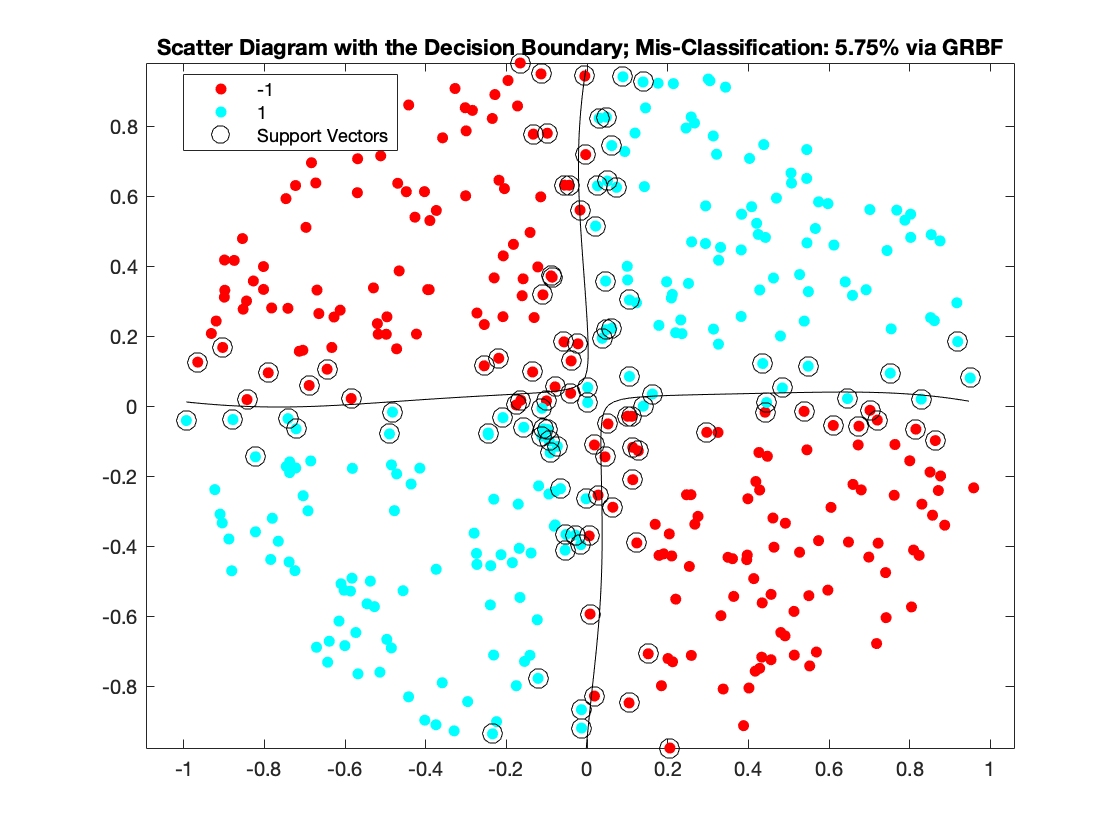}
  \caption{{[\tt{MATLAB}]} SVM classifier via different kernels.}\label{fig:SVMResultsGGRBF}
\end{figure}
Out of three kernels, \eqref{eq_GGRBF} yields the lowest miss-classification for the data of $100$ sampled points.
\subsection{Neural Network}
Following examples provide optimally-best results by the \eqref{eq_GGRBF} employment both as in the usage of \emph{Activation Function} and neural net layer in \emph{deep convolutional neural nets} (DCNN).
\subsubsection{Activation Function}
Consider an $\alpha$ReLU activation function, a simple modification of powerful-traditional ReLU activation function from NN, defined as
\begin{align}
    f(x)\coloneqq\begin{cases} x&\text{if $x>0$}\\
    \alpha x&\text{if $x\leq0$}.\tag{
    \text{$\alpha$}ReLU}\label{Modified ReLU}
    \end{cases}
\end{align}
In the present experiment, a two 7-layered NNs are constructed; one with the activation function defined by \eqref{Modified ReLU} and one with \eqref{eq_GGRBF}. Thereafter, following performance tables yields the respective results and comparison for the NNs.
\begin{figure}[H]
  \centering
  \label{fig:nn_GGRBF}\includegraphics[width=.45\linewidth,scale=1]{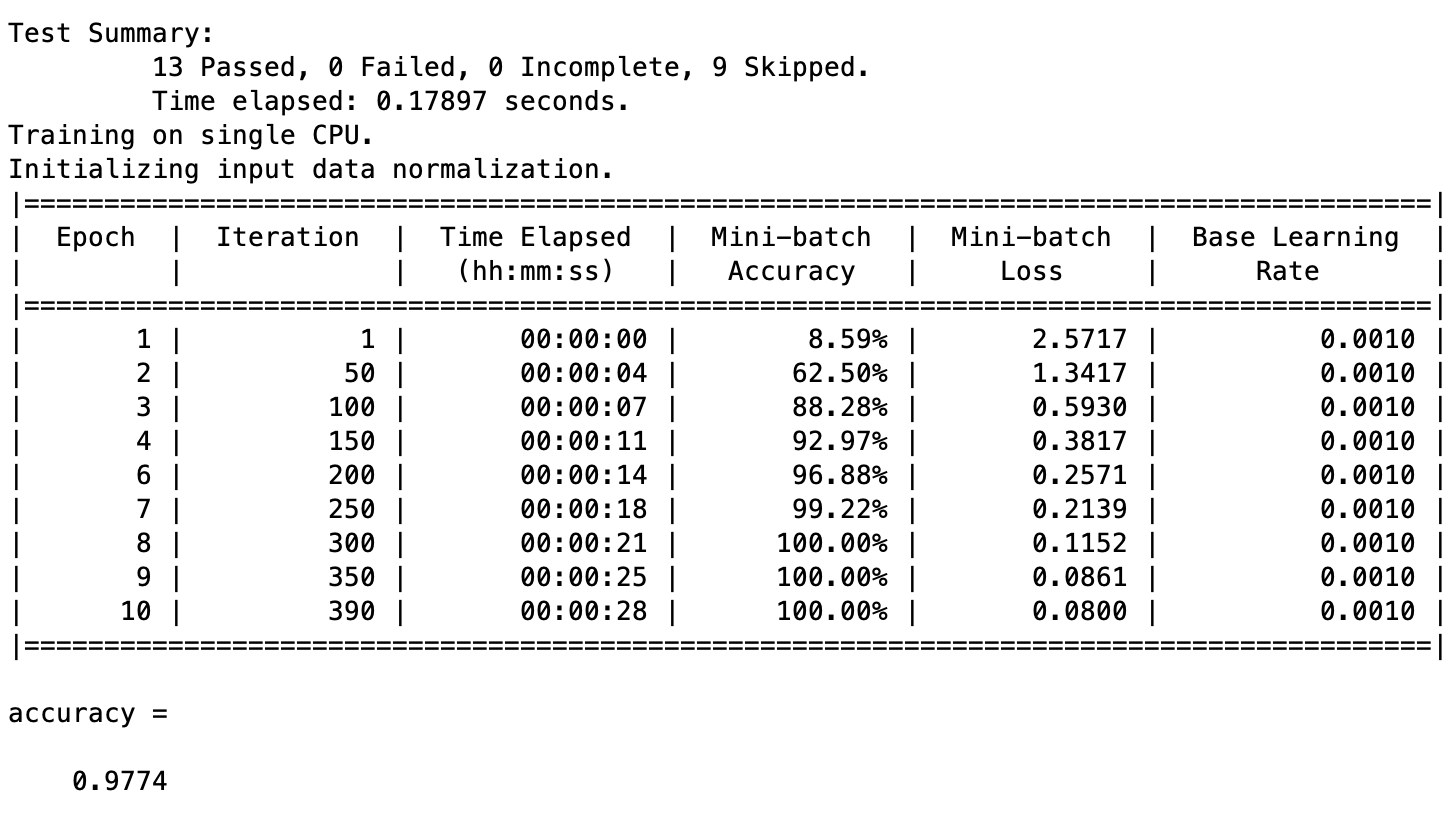}\hfill
  \includegraphics[width=0.45\linewidth,scale=1]{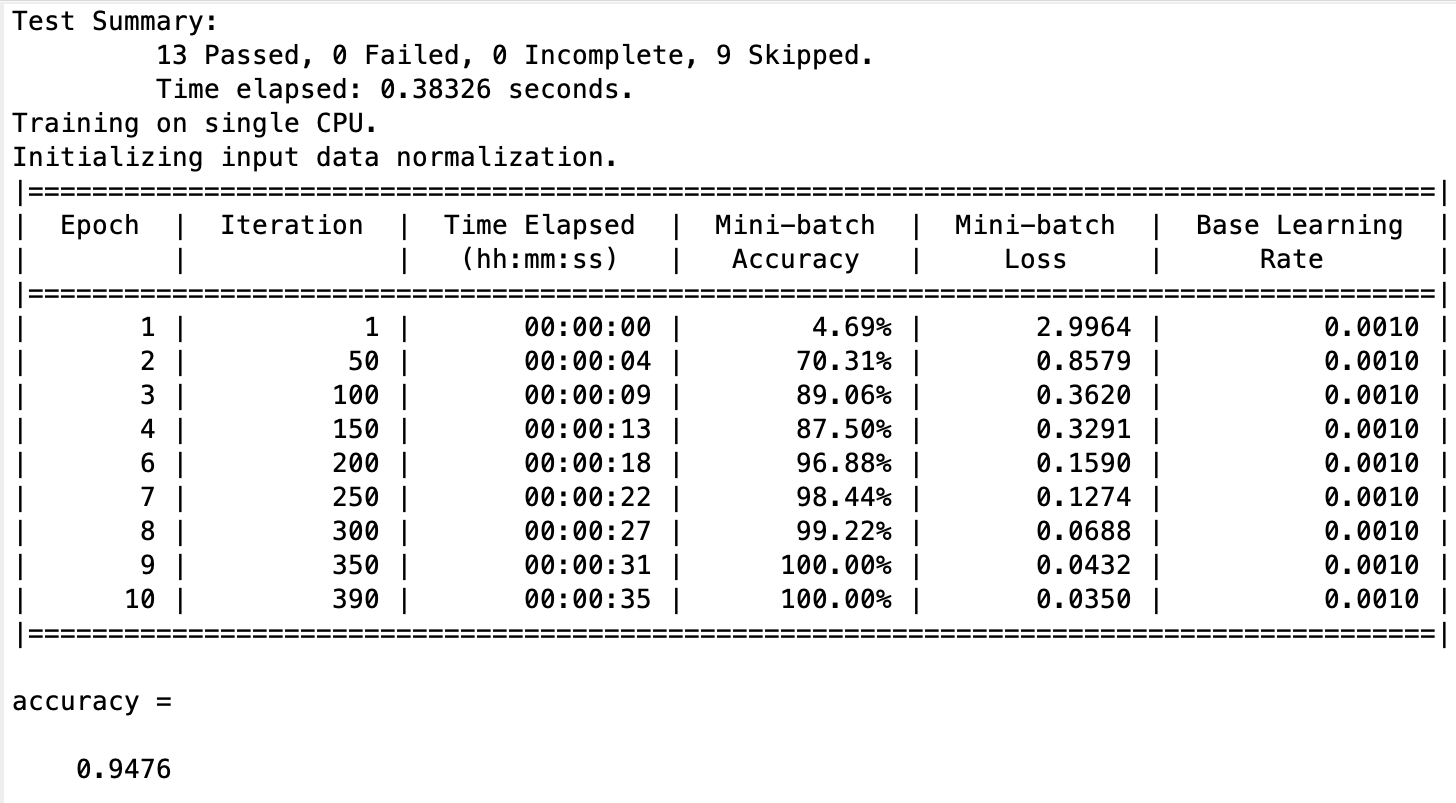}
  \caption{[{\tt{MATLAB}]} Performance comparison \eqref{eq_GGRBF} \emph{(L)} \& \eqref{Modified ReLU} \emph{(R)} as the NN activation function.}
  \label{fig:nn_noGGRBF}
\end{figure}
{\tt{MATLAB}} script for the activation function of GGRBF is provided in \autoref{section_matlab_code}.
\subsubsection{DCNN}
Above experiments demonstrate the concrete functionality of \eqref{eq_GGRBF} as an activation function, it further motivates to perform DCNN
. Therefore, in pursue of this, a typical DCNN of $7$-layered is constructed; one with the activation function defined by \eqref{Modified ReLU} and one with \eqref{eq_GGRBF}. In the following experiment, the training data contains $1500$-$28\times28$ gray-scale letter images of $A$, $B$, and $C$ in a $4$-D array.
\begin{figure}[H]
  \centering
  \label{fig:dcnn_GGRBF}\includegraphics[width=.345\linewidth,scale=1]{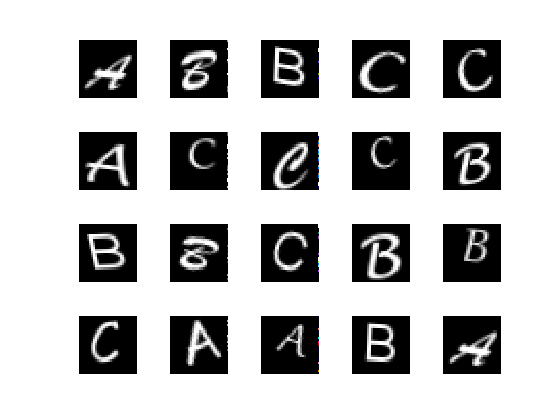}\hfill
  \includegraphics[width=0.345\linewidth,scale=1]{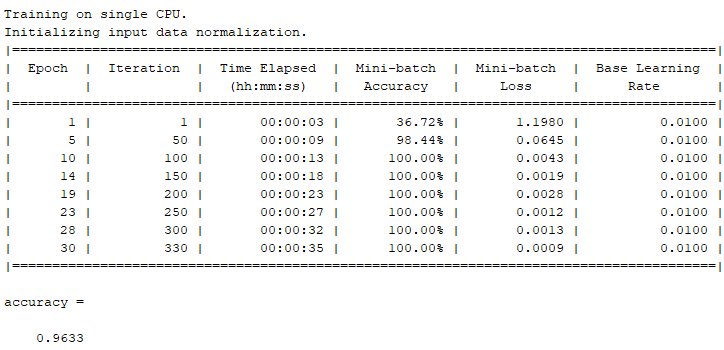}
  \caption{{[\tt{MATLAB}]} $96.33\%$ accuracy registered with \eqref{eq_GGRBF} as DCNN neural layer.}
  \label{fig:dcnn_GGRBF_performance}
\end{figure}
\begin{figure}[H]
  \centering
  \label{fig:dcnn_noGGRBF}\includegraphics[width=.345\linewidth,scale=1]{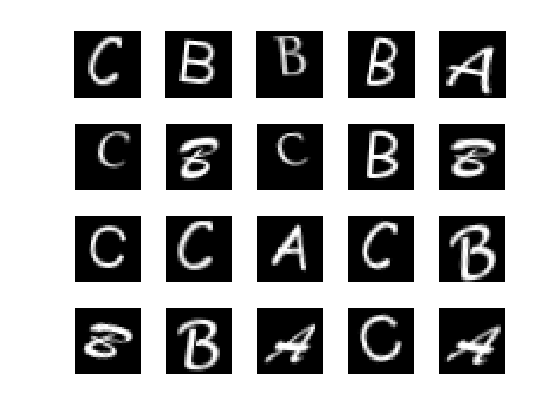}\hfill
  \includegraphics[width=0.345\linewidth,scale=1]{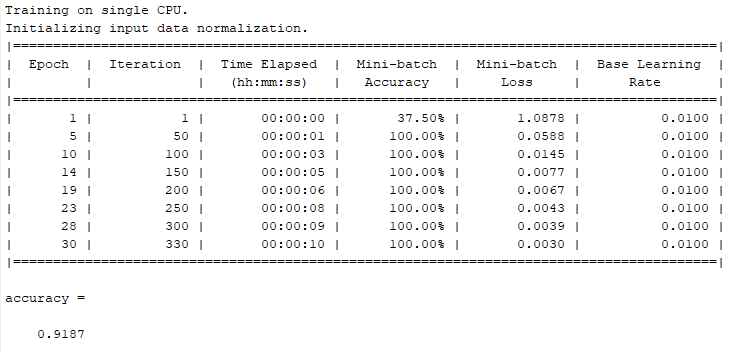}
  \caption{{[\tt{MATLAB}]} $91.87\%$ accuracy registered with \eqref{Modified ReLU} as DCNN neural layer.}
  \label{fig:dcnn_noGGRBF_performance}
\end{figure}
\section{Future Directions}
\subsection{Eigen-function expansion of GGRBF}
We recall the \emph{Mercer's Theorem} from \protect{\cite[Theorem 4.2, Page 96]{williams2006gaussian}}.
\begin{theorem}[Mercer's Theorem]\label{theorem_Mercer}
    Let $\left(\mathcal{X},\mu\right)$ be a finite measure space and $k\in L_\infty\left(\mathcal{X}^2,\mu^2\right)$ be a kernel such that $T_k:L_2\left(\mathcal{X},\mu\right)\to L_2\left(\mathcal{X},\mu\right)$ is positive definite. Let $\left\{\phi_i\right\}_i\in L_2\left(\mathcal{X},\mu\right)$ be the normalized eigenfunctions of $T_k$ associated with the eigenvalues $\left\{\Lambda_i\right\}_i$. Then:
    \begin{enumerate}
        \item the eigenvalues $\left\{\Lambda\right\}_i$ are absolutely summable
        \item \begin{align}
            k\left(\mathbf{x},\mathbf{x'}\right)=\sum_{i=0}^\infty\Lambda_i\phi_i\left(\mathbf{x}\right)\phi_i\left(\mathbf{x'}\right)^*
        \end{align}
        holds $\mu^2$ almost everywhere, where the series converges absolutely and
uniformly $\mu^2$ almost everywhere.
    \end{enumerate}
\end{theorem}
\begin{example}
    With the application of \autoref{theorem_Mercer}, we can provide the eigen-function decomposition of $K_\sigma\left(x,z\right)=e^{-\sigma^2\left(x-z\right)^2}$ for $x,z\in\mathbf{R}$, that is: 
    \begin{align}
        e^{-\sigma^2\left(x-z\right)^2}=&\sum_{i=0}^\infty\Lambda_i\phi_i\left(x\right)\phi_i\left(z\right),\quad\text{where}\\
        \Lambda_i=&\frac{\alpha\sigma^{2i}}{\left(\frac{\alpha^2}{2}\left(1+\sqrt{1+\left(\frac{2\sigma}{\alpha}\right)^2}\right)+\frac{\sigma^2}{2}\right)^{i+\nicefrac{1}{2}}}\\
        \phi_i\left(x\right)=&\frac{\sqrt[8]{1+\left(1+\frac{2\sigma}{\alpha}\right)^2}}{\sqrt{2^ii!}}e^{-\left(\sqrt{1+\left(\frac{2\sigma}{\alpha}\right)^2}-1\right)\frac{\alpha^2x^2}{2}}
        H_i\left(\sqrt[4]{1+\left(\frac{2\sigma}{\alpha}\right)^2}\alpha x\right)\label{eq_20Hi}.
    \end{align}
    The expression $\left\{H_i\left(\bullet\right)\right\}_i$ in \eqref{eq_20Hi} are the \emph{Hermite polynomials} which are $L_2-$orthonormal against the weight $\nicefrac{\alpha}{\pi}e^{-\alpha^2x^2}$; that is:
    \begin{align}
        \int_{\mathbf{R}}\phi_n(x)\phi_m(x)\frac{\alpha}{\pi}e^{-\alpha^2x^2}dx=\delta_{nm}.
    \end{align}
\end{example}
We have the graphical representation of first seven Hermite polynomials in following figures \autoref{fig:first seven Hermite polynomials}: 
\begin{figure}[H]
    \centering
    \includegraphics[scale=.9]{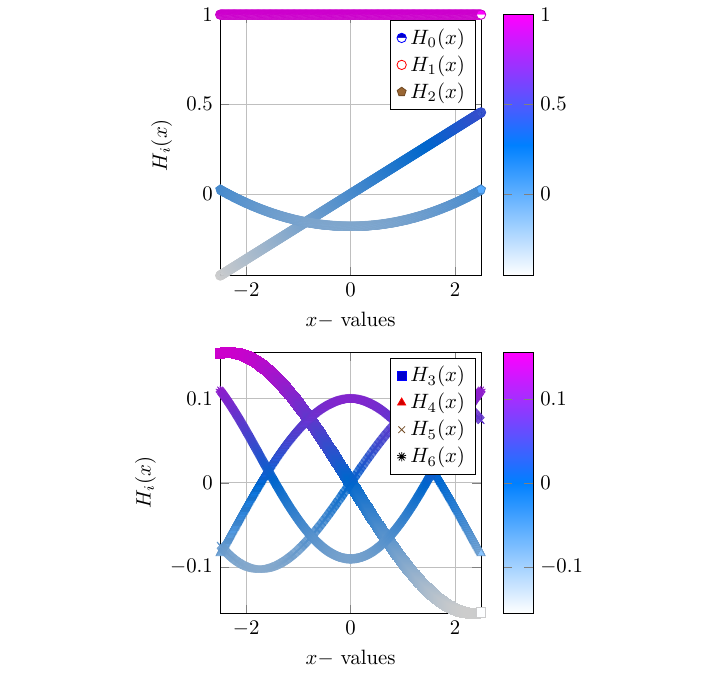}
    \caption{First seven Hermite polynomials}
    \label{fig:first seven Hermite polynomials}
\end{figure}
Based on the Hermite polynomials \protect{\cite{hermite1864nouveau}} and as its application for the eigen decomposition analysis for GRBF kernel, we have the following promising future research direction.

In the spirit of the application of the Mercer's Theorem, we know the eigen-function decomposition of the GRBF Kernel. However, presently, we are not fortunate to have such a decomposition for the GGRBF Kernel. A preliminary investigation towards the desired eigen-function decomposition of GGRBF Kernel (followed from \protect{\cite{rasmussen2006gaussian,zhu1997gaussian,fasshauer2012stable}}) directs us to incorporate a new variety of function defined in \eqref{eq_22GGRBF}.
\begin{align}
    \mathcal{H}_n(x)
    \coloneqq&(-1)^ne^{ax^2}e^{-e^{-bx^2}+1}\frac{d^n}{dx^n}\left(e^{-ax^2}e^{e^{-bx^2}-1}\right)\label{eq_22GGRBF}
\end{align}
into our desired eigen-function decomposition analysis. Here $a>0$ and $b\geq0$ and therefore, if $a=1~\&~b=0$ then $\mathcal{H}_i=H_i$. The first two expression for $\mathcal{H}_n(x)$ are explicitly given as:
\begin{align*}
    \mathcal{H}_1(x)=&2ax+2bxe^{-bx^2},\tag{D1}\label{eq_D1}\\
    \mathcal{H}_2(x)=&-2a+4b^2x^2e^{-bx^2}-2be^{-bx^2}+\left(2ax+2bxe^{-bx^2}\right)^2\tag{D2}\label{eq_D2}.
\end{align*}
The constants $a$ and $b$ present in \eqref{eq_22GGRBF} corresponds to the respective constants present in \eqref{eq_GGRBF}; in particular $a={\sigma}^2$ and $b={\sigma}_0^2$. Following are the graphical presentation of first seven function from the family defined in \eqref{eq_22GGRBF}.
\begin{figure}[H]
    \centering
    \includegraphics[scale=.55]{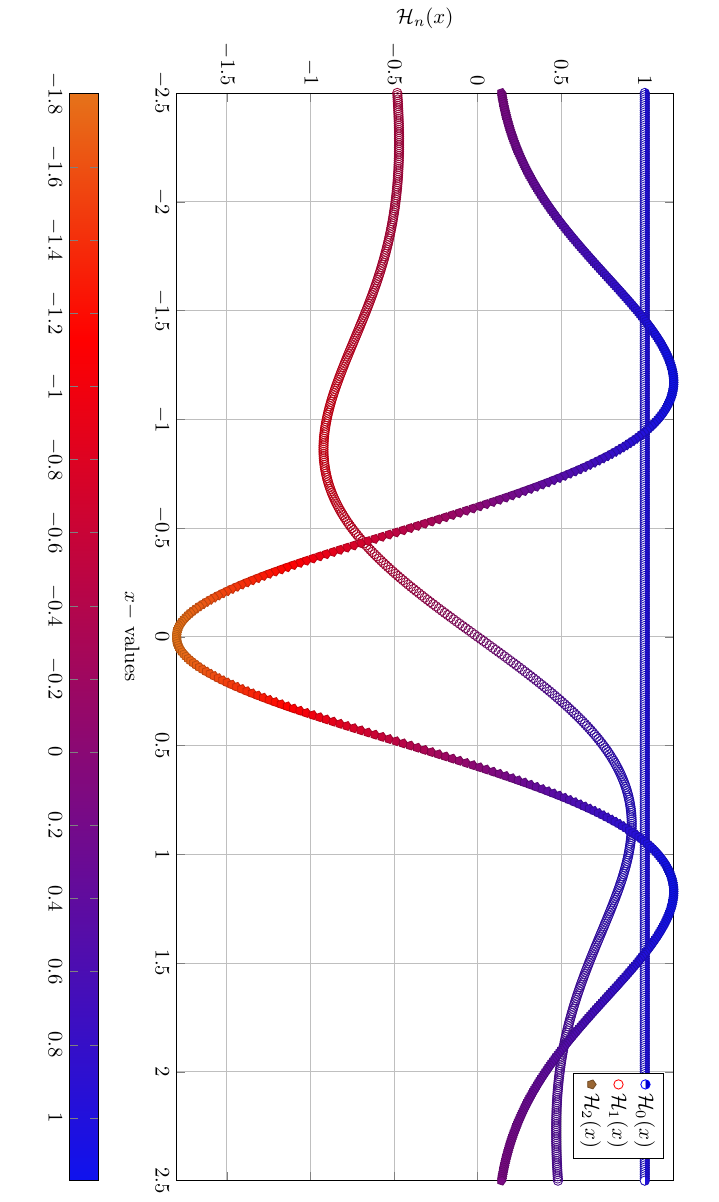}
    \includegraphics[scale=.55]{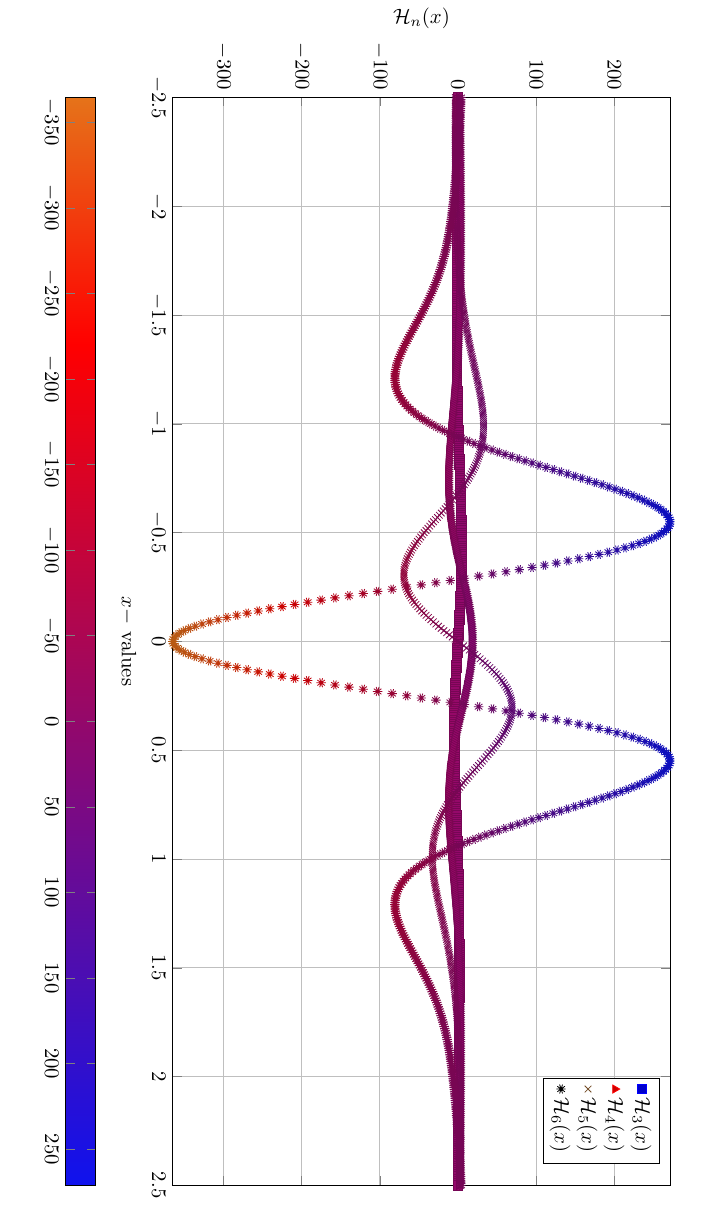}
    \caption{Graph of $\mathcal{H}_n(x)$ for $n=0,1,2,3,4,5$ and $6$ with $a=.091~\&~b=0.81$.}
    \label{fig:Graph_Hermite_like_1234567}
\end{figure}
Following \textsc{\autoref{my-label}} is the compilation of the results documented for the various experiments we discussed so far.
{{
\begin{table*}[ht]
\caption{\text{Compilation of results}}
\label{my-label}
\begin{tabularx}{\textwidth}{@{} l p{.89in} *{10}{C} c @{}}
\toprule
{\tiny{\textsc{AI Learning Architecture}}} & {\tiny{\textsc{Mathematical Function}}} & {\tiny{\textsc{Figure Ref.}}} 
& {\tiny{\textsc{Minimum Error}}}  
& {\tiny{\textsc{Misclass. \%}}} & {\tiny{\textsc{Accuracy \%}}}
\\ 
\midrule
Kernel Regression & GRBF   &\autoref{fig:regression-2} &0.0023      
& -   & - 
\\ 
Kernel Regression&\textbf{GGRBF} &\autoref{fig:regression-2} &$\bm{9.6913\times10^{-4}}$  
& -    & -\\
Kernel Regression & GRBF   &\autoref{fig:regression-4} 
& 0.0010  & -   & - 
\\ 
Kernel Regression&\textbf{GGRBF} &\autoref{fig:regression-4} 
& $\bm{4.2882\times10^{-4}}$ & -    & -
\\ 
\midrule
Support Vector Machine&GRBF&\autoref{fig:SVMResultsGGRBF}&-
&5.75&94.25\\
Support Vector Machine&Sigmoid&\autoref{fig:SVMResultsGGRBF}&-
&4.5&95.5\\
Support Vector Machine&\textbf{GGRBF}&\autoref{fig:SVMResultsGGRBF}&-
&$\bm{3.75}$&$\bm{96.25}$\\
\midrule
Activation Function Neural Network&$\alpha$ReLU&\autoref{fig:nn_noGGRBF}&-
&5.24&{94.76}\\
Activation Function Neural Network&\textbf{GGRBF}&\autoref{fig:nn_noGGRBF}&-
&$\bm{2.26}$&$\bm{97.74}$\\
\midrule
Deep Convolutional Neural Network&\textbf{GGRBF}&\autoref{fig:dcnn_GGRBF_performance}&-
&$\bm{3.67}$&$\bm{96.33}$\\
Deep Convolutional Neural Network&$\alpha$ReLU &\autoref{fig:dcnn_noGGRBF_performance}&-
&8.13&{91.87}\\
\bottomrule
\end{tabularx}
\end{table*}}}
\begin{futuredirection}
{Having stated that, we are still in the void knowledge for eigen-values of the GGRBF Kernel. Additionally, we still need to investigate whether the function introduced in \eqref{eq_22GGRBF} are orthonormal in the sense as the traditional Hermite polynomials are. Therefore, it will be interesting to understand the analysis of the function given in \eqref{eq_22GGRBF}.}
\end{futuredirection}
\subsection{Operator Theory Analysis for Data-Driven Problems}
Modern data driven problems arising in the filed of dynamical systems are captured by the Liouville Operators \protect{\cite{rosenfeld2022dynamic}}, Liouville Weighted Composition Operators \protect{\cite[Chapter 2]{singh2023applied}} or Koopman Operators \protect{\cite{williams2015data}} acting over the underlying Hilbert spaces. The work-horse algorithm in the direction of \emph{reduced order modelling} (ROM) techniques is \emph{Dynamic Mode Decomposition} (DMD) \protect{\cite{schmid2010dynamic}} which aims to determine the spatio-temporal coherent structures of high-dimensional time-series data is executed by extracting the eigen-observables of the aforementioned operators over the underlying RKHS.

These RKHS are \emph{Bergman-Seigal-Fock Space} \protect{\cite{zhu2012analysis}} which is generated by the \emph{exponential dot product kernel}. The $L^2-$ measure for this RKHS is (normalized) Gaussian measure $\left(\nicefrac{\sigma^2}{\pi}\right)^de^{-\sigma^2|\bm{z}|^2}dV_{\mathbf{C}^d}(\bm{z})$. Following are the results from the DMD experiment for the vorticity of the fluid flow across the cylinder when the chosen kernel was the GRBF Kernel.
\begin{figure}[H]
    \centering
    \frame{
    \includegraphics[
    scale=0.34]{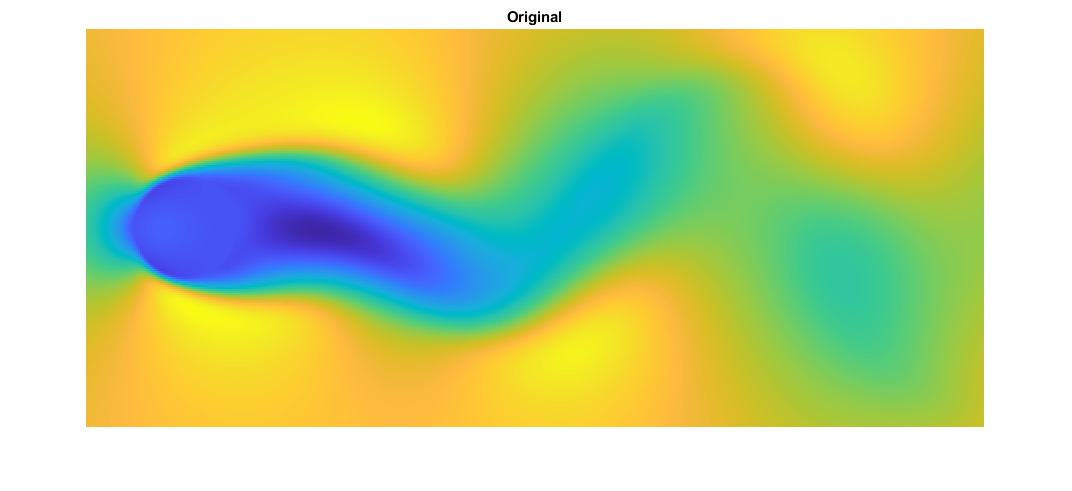}}
    \caption{Original DMD Experiment for the fluid flow across the cylinder. 
    }
    \label{fig:original}
\end{figure}
\begin{figure}[H]
    \centering
    \frame{\includegraphics[
    scale=.34]{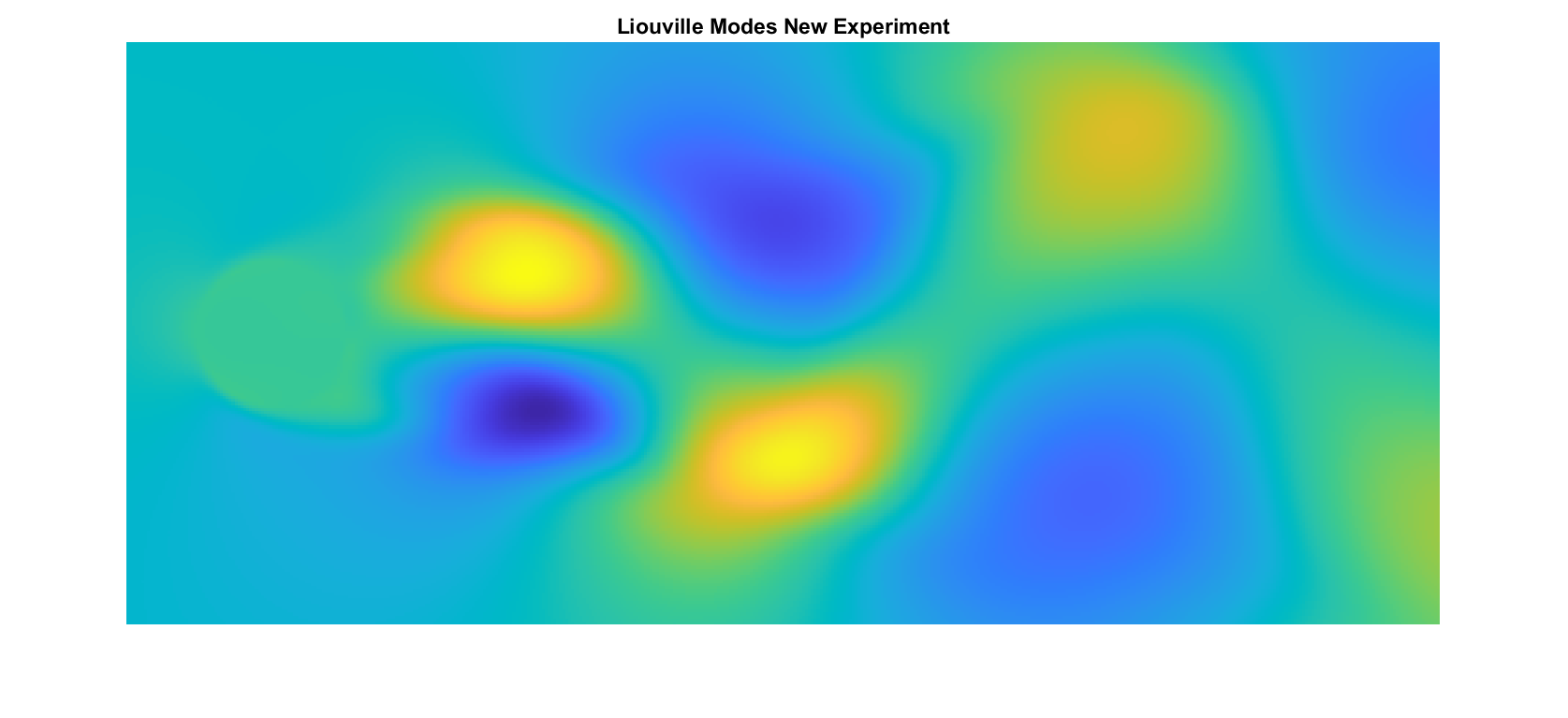}}
    \caption{199\textsuperscript{th}-Liouville Mode (no noise) via the GGRBF Kernel for the fluid flow across the cylinder. 
    }
    \label{fig:my_label_Liouville Modes_GGRBF}
\end{figure}
\begin{figure}[H]
    \centering
    \frame{\includegraphics[
    scale=.34]{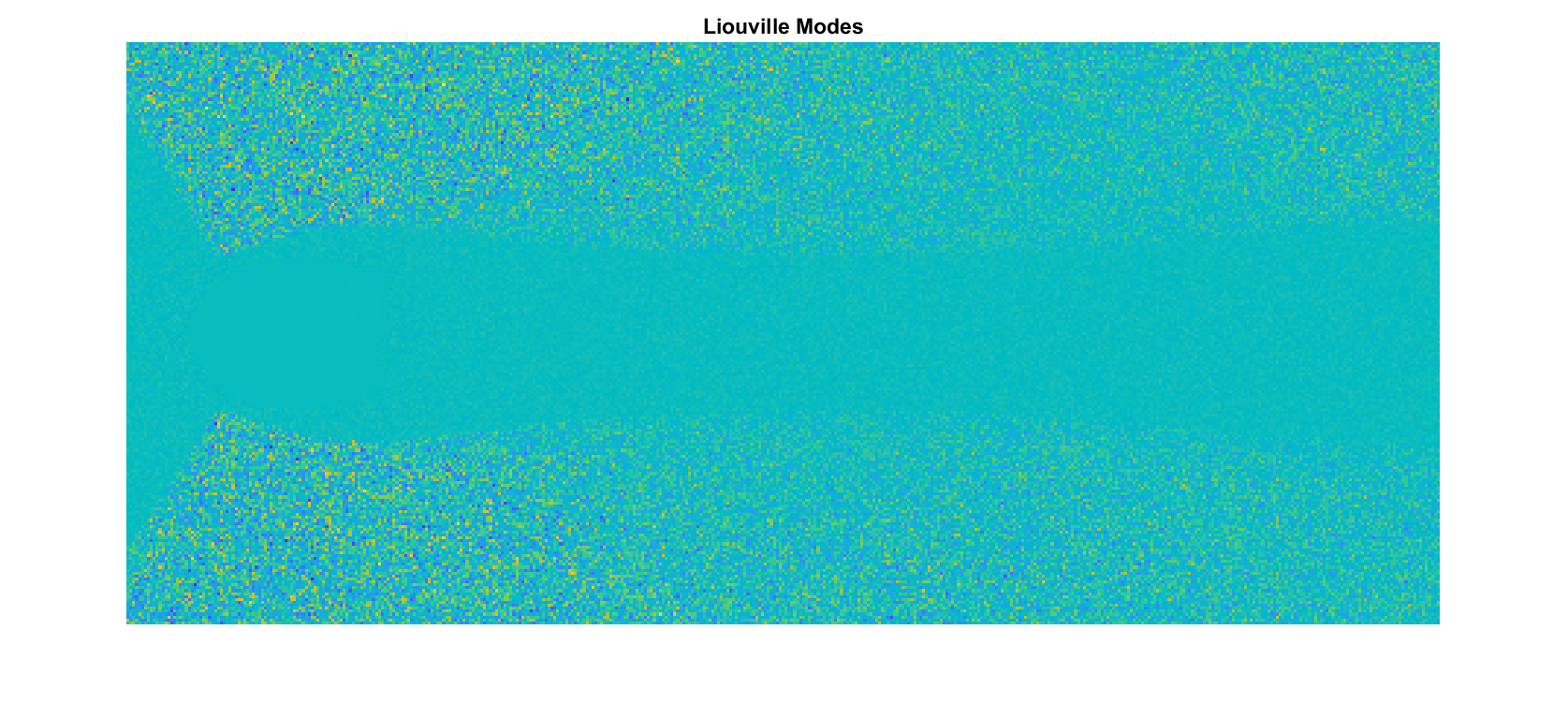}}
    \caption{199\textsuperscript{th}-Liouville Mode (noise) via the GRBF Kernel for the fluid flow across the cylinder. 
    }
    \label{fig:my_label_Liouville Modes_GRBF}
\end{figure}
Clearly, DMD results by GGRBF in \autoref{fig:my_label_Liouville Modes_GGRBF} contains visibly-no-noise as compared to results obtain by GRBF in \autoref{fig:my_label_Liouville Modes_GRBF}.
\begin{futuredirection}
We understand the action of Koopman Operators (or composition operators) over the Bergman-Seigal-Fock Space due to the investigation performed by Carswell, MacCluer and Schuster in \protect{\cite{carswell2003composition}}. On the other hand, \protect{\cite[Theorem 2.25, Chapter 2]{singh2023applied}} demonstrates the provable convergence phenomena for dynamical systems by the Liouville weighted composition operators over the Bergman-Seigal-Fock Space. It will be interesting to carry-out the similar operator theoretic investigation over the RKHS introduced in \eqref{eq_8hilbertspace}.    \end{futuredirection}
\section{Acknowledgement}
The author would like to thanks reviewer for their precious time on reviewing the present manuscript. In addition to this, author also acknowledges the support of following as well.
\begin{enumerate}
\item The author acknowledges the support of \textit{Ms. Drishty Singh}, \emph{4th Year MSc, Department of Mathematics, Babasaheb Bhimrao Ambedkar University, Lucknow, Uttar Pradesh, India}. She provided the explicit expression of first two functions in \eqref{eq_D1} and \eqref{eq_D2} from \eqref{eq_22GGRBF}. Further analysis of the function in \eqref{eq_22GGRBF} was extended due to these important results.
\item The author is thankful for the valuable discussion with \text{Dr. Romit Maulik}, \emph{Assistant Professor, Information Sciences and Technology, Institute of Computational and Data Sciences, Pennsylvania State University}. The results in \autoref{section_comparisonresults} were discovered during the research discussion with him in Fall 2022's ending for preparing the research proposal for \emph{Eric and Wendy Schmidt AI in Science Postdoctoral Fellow Program-2023}.
\item The empirical results in \autoref{section_comparisonresults} were discovered during the last year of PhD at University of South Florida, where he was supported by \textsc{Dr. Joel A. Rosenfeld}'s (PhD advisor) NSF and AFOSR grants (AFOSR Young Investigator Research Program (YIP) Award FA9550-21-1-0134, FA9550-20-1-0127 \& ECCS-2027976). The author would like to thanks his advisor's support on this. The RKHS theory present in the first half of this paper was developed while he was at The University of Texas at Tyler.
\end{enumerate}
\section{{\tt{MATLAB}} Script}\label{section_matlab_code}
\subsection{Neural Net Layer of GGRBF}In order to execute experiments which employs the neural net layer of activation function as GGRBF, we have to construct it from the scratch. In that regards, following is the {\tt{MATLAB}} script for the construction of custom neural net layer for the GGRBF.
\lstset{language=Matlab,%
    frame=single,
    basicstyle=\color{red},
    breaklines=true,%
    morekeywords={matlab2tikz},
    keywordstyle=\color{blue},%
    morekeywords=[2]{1}, keywordstyle=[2]{\color{black}},
    identifierstyle=\color{black},%
    stringstyle=\color{mylilas},
    commentstyle=\color{mygreen},%
    showstringspaces=false,
    numbers=left,%
    numberstyle={\tiny \color{black}},
    numbersep=9pt, 
}
\begin{lstlisting}[style=Matlab-editor]
classdef ggrbf < nnet.layer.Layer
    % Custom GGRBF layer.
   
    properties (Learnable)
        % Layer learnable parameters.
        % Scaling coefficients.
        Alpha
        Beta
    end
   
    methods
        function layer = ggrbf(numChannels, name)
            
           
            % Set layer name.
            layer.Name = name;
           
            
            % Initialize scaling coefficient.
            layer.Alpha = rand([1 1 numChannels]);
            layer.Beta = rand([1 1 numChannels]);
        end
       
        function Z = predict(layer, X)
        % Z = predict(layer, X) forwards the input data X through the layer and outputs the result Z.
            Z=exp(-(layer.Alpha).^(-2).*X.*X).*exp(exp(-(layer.Beta).^(-2).*X.*X)-1);
        end
    end
end
\end{lstlisting}
\newcommand{\etalchar}[1]{$^{#1}$}

\end{document}